\documentclass{article} 
\usepackage{amsmath,amssymb,amsthm}
\usepackage{fullpage}
\usepackage{url}
\usepackage[numbers]{natbib}
\usepackage{color}
\usepackage{graphicx}

\usepackage{algorithm,algorithmic}
\usepackage[usenames,dvipsnames]{xcolor}
\usepackage[colorlinks=true,linkcolor=blue,citecolor=violet]{hyperref}

\title{Hierarchies of Relaxations for Online Prediction Problems with Evolving Constraints}
\author{Alexander Rakhlin \\ University of Pennsylvania \and Karthik Sridharan \\ Cornell University}

\newtheorem{theorem}{Theorem}
\newtheorem{lemma}[theorem]{Lemma}

\newtheorem{corollary}[theorem]{Corollary}

\newtheorem{proposition}[theorem]{Proposition}
\newtheorem{example}{Example}

\theoremstyle{definition}
\newtheorem{definition}{Definition}

\newcommand{\Las}{\mathcal{L}\mathrm{as}}

\newcommand{\mbb}[1]{\mathbb{#1}}
\newcommand{\mbf}[1]{\mathbf{#1}}
\newcommand{\mc}[1]{\mathcal{#1}}
\newcommand{\mrm}[1]{\mathrm{#1}}

\newcommand{\V}{\mc{V}}
\newcommand{\gap}{\mrm{gap}}
\newcommand{\pred}{\widehat{{y}}}

\newcommand{\norm}[1]{\left\|#1\right\|}

\newcommand{\argmin}[1]{\underset{#1}{\mrm{argmin}} \ }

\newcommand{\reals}{\mathbb{R}}
\newcommand{\E}[1]{\mathbb{E}\left[ #1 \right]} 
\newcommand{\En}{\mathbb{E}}  
\newcommand{\Es}[2]{\mathbb{E}_{#1}\left[ #2 \right]} 

\newcommand{\inner}[1]{\left\langle #1 \right\rangle}
\newcommand{\ip}[2]{\left<#1,#2\right>}

\newcommand{\ind}[1]{{\bf 1}\left\{#1\right\}}
\newcommand{\tr}{\ensuremath{{\scriptscriptstyle\mathsf{T}}}}

\newcommand{\bp}{\boldsymbol{p}}
\newcommand{\bepsilon}{\boldsymbol{\epsilon}}

\newcommand\y{\mathbf{y}}

\newcommand\cI{\mathcal{I}}

\newcommand\X{\mathcal{X}}
\newcommand\Y{\mathcal{Y}}

\newcommand\F{\mathcal{F}}
\newcommand\G{\mathcal{G}}

\newcommand\M{\mathcal{M}}

\newcommand\Reg{\mbf{Reg}}

\newcommand\loss{\boldsymbol{\ell}}
\newcommand{\Relax}[3]{\mbf{Rel}_{#1}\left(#2 ~\middle| #3 \right)}

\newcommand{\Fclass}[2]{\F_{#1}\boldsymbol{[}#2\boldsymbol{]}}
\newcommand{\constr}{{\mathscr{C}}}
\def\deq{\triangleq}

\newcommand{\frameit}[1]{
    \begin{center}
    \framebox[1\columnwidth][c]{
        \begin{minipage}{.99\columnwidth}
        #1
		\vspace{-0mm}
        \end{minipage}
    }
    \end{center}
}

\usepackage{bbm,mathrsfs}

\begin{document}

\maketitle

\begin{abstract}
		We study online prediction where regret of the algorithm is measured against a benchmark defined via evolving constraints. This framework captures online prediction on graphs, as well as other prediction problems with combinatorial structure. A key aspect here is that finding the optimal benchmark predictor (even in hindsight, given all the data) might be computationally hard due to the combinatorial nature of the constraints. Despite this, we provide polynomial-time \emph{prediction} algorithms that achieve low regret against combinatorial benchmark sets. We do so by building improper learning algorithms based on two ideas that work  together. The first is to alleviate part of the computational burden through random playout, and the second is to employ Lasserre semidefinite hierarchies to approximate the resulting integer program. Interestingly, for our prediction algorithms, we only need to compute the values of the semidefinite programs and not the rounded solutions. However, the integrality gap for Lasserre hierarchy \emph{does} enter the generic regret bound in terms of Rademacher complexity of the benchmark set. This establishes a trade-off between the computation time and the regret bound of the algorithm.
\end{abstract}


\section{Introduction}

To motivate the general setting of the paper, let us start with an example. Consider the problem of node label prediction in an evolving social network. At each round, a new user joins the network and makes connections to some existing users. The observable part of a user's type is represented by a covariate vector (or, \emph{side information}) that may consist of gender, age, education level, and other revealed  characteristics. Suppose we are tasked with developing a system that predicts a ``label'' for the user, in a possible set of outcomes. For instance, our goal might be to conduct a successful marketing campaign; here, the unseen labels could stand for the type of product the user will buy. Having made the prediction, we observe the actual behavior of the person (such as a purchase) and suffer a cost if the prediction was wrong. 

We would like to devise a framework for developing prediction algorithms for this problem. Several aspects require careful consideration. First, how do we phrase the goal of the forecaster? Second, how do we model the evolution of the graph, arrival of users, and users' covariate vectors? Third, how can we leverage global information dispersed in the network in order to make good predictions on the individual level? Last but not least, how do we develop computationally feasible prediction methods? 

To make matters concrete, consider an example where at each time step $t$ a new user   joins the network, and the links (edges) to other users are revealed along with side information $x_t$ about the user. We may think of the weights $W_{ij}\in[-1,1]$ as the strength of similarity (dissimilarity) between users $i$ and $j$. This number is only known if $i,j\leq t$. The system makes a binary prediction $\pred_t$, and the actual label $y_t$ of the user is subsequently revealed. For developing such a prediction system, the practitioner would need to incorporate prior knowledge about the problem. For instance, it might be reasonable to assume that at the end of $V$ rounds, the nodes of the graph will be roughly clustered in terms of their labels, with within-community links being mostly positive and across-community links being mostly negative. In addition to this adherence of labels to the graph structure, we also encode prior information through a function class $\F$ of mappings from side-information to labels. For instance, in binary classification it might be reasonable to suspect a linear separation between the two classes in terms $\text{sign}(w\cdot x_t)$ for some $w$. Unfortunately, the connectivity, side information, and the labels are only partially known until the end of $V$ rounds. Nevertheless, we set the goal as that of predicting as well as if this information were available: the performance is measured by the regret
$$\sum_{t=1}^V \ind{\pred_t\neq y_t} - \inf_{f\in\F[\text{data}]}\sum_{t=1}^V \ind{f(x_t)\neq y_t},$$
where $\F[\text{data}]\subseteq\F$ is only known at the end of $V$ rounds (precise definition given in the next section). $\F[\text{data}]$ is a data-dependent set of labelings that (we hope) models well the prediction problem at hand (see \citep{cesa2013random} and references therein for related graph prediction problems). 

Given the interpretation that positive $W_{ij}$'s encode similarity and negative $W_{ij}$'s encode dissimilarity, it is natural to let $\F$ be a set of labelings such that the number of disagreements at endpoints is minimized for edges with positive weights and maximized for edges with negative weight. This smoothness of $f\in\F$ with respect to the graph can be encoded by the graph Laplacian $L$, and one can   use $\F=\{f\in\{\pm1\}^V: f^\tr L f\leq K\}$, for some parameter $K>0$  \citep{RakSri14chervonenkis}. The authors of the latter paper proposed a straightforward relaxation to obtain a computationally feasible method, at the expense of having a larger regret bound. This is a starting point for the present paper.

We depart from the usual regret minimization framework in several ways. First, instead of restricting the set of possible labelings based solely on the graph structure and edge weights, we model the set $\F$ through the number of satisfied constraints. To this extent, a graph structure is just a particular set of constraints that involve \emph{pairs} of nodes (which we shall interchangeably call ``items'' or ``individuals''). A more general constraint might involve groups of individuals, and this gives greater flexibility in modeling the overall interaction between the nodes. Formally, a constraint is an arbitrary binary or real-valued function from assignments of labels for a subset of nodes to $\reals_{\geq 0}$.  Within theoretical computer science, constraint satisfaction problems (CSPs) are a natural umbrella for such combinatorial problems as {\sf Max Cut}, {\sf Unique Games}, and {\sf Max $k$-SAT}. Furthermore, under the Unique Games Conjecture, semidefinite relaxations are providing an optimal approximation ratio for every CSP \citep{raghavendra2008optimal,raghavendra2009round}. One of the goals of his paper is to apply semidefinite relaxation techniques to the problem of online prediction with combinatorial constraints.

The second way in which we depart from the traditional work on online learning is in allowing constraints to be revealed in an online manner. For the example of a graph-based constraints, this means that the graph can be revealed to the forecaster sequentially. Moreover, we can think of the graph as \emph{evolving} in time since identities of the nodes have little significance, except for being arguments to constraints. We assume that the probability distribution that governs this evolution is known to the forecaster. As a particular case, the distribution may put all the mass on the revelation of all the constraints at the first round, in which case the constraints (or, the graph) are ``known ahead of time.'' More generally, one may take graph evolution models studied in probability theory and in social networks research, and use these for the prediction problem. In addition to the evolution of constraints, we allow the forecaster to observe side information about the new node. This side information is, once again, stochastic and follows a distribution jointly with constraints and node identities. 

While the constraints and side information are stochastic, the label is chosen in an adversarial way. We have in mind the situation where we can model the network structure and the distribution of people types, but the label (or, action) of the person is not easily modeled. Instead, this behavior can be best understood through \emph{global information} within the network, not the local information. Such a global coherence of labels and the constraints is modeled through the comparator class $\F$.

It would appear that the overall framework involving constraints, side information, and adversarially chosen labels cannot yield computationally tractable algorithms. Yet we show that by moving to improper prediction algorithms one can develop computationally efficient methods for the problem with only slight worsening of the regret guarantees. As a first step towards developing efficient methods, we show that the knowledge of the overall distribution governing the presentation of constraints and the side information allows us to define a randomized method with a provable guarantee on prediction error. We analyze ``random playout,'' a method that simulates future constraints and side information and uses these hallucinated values in place of missing information. We show that such an algorithm (which arises from the relaxation framework in \citep{rakhlin2012relax}) has regret that is bounded by classical Rademacher complexity of $\F$ given the constraints and side information. 

The last missing piece in this story is how to calculate the next prediction given the random playout. Here, we show that the forecaster needs to compute a value with conditional Rademacher complexity as part of the objective. In general, the computation of Rademacher complexity is not a feasible task for the types of combinatorial constraints we have in mind. However, the online relaxation framework suggests that we may take a superset of $\F$ (given the constraints and side information) and suffer regret of Rademacher complexity of this larger set. We propose to use semidefinite hierarchies for this task. In particular, we define Lasserre hierarchy \citep{lasserre2001global,parrilo2003semidefinite} to obtain polynomial-time prediction methods with a ``knob'' (level of the hierarchy) that trades off computational time and prediction performance as measured by the regret. 

In this paper, two distinct uses of the word ``relaxation'' come together. \emph{Online relaxations} are upper bounds on the minimax value of the multistage prediction problem \citep{rakhlin2012relax}. One of a number of approaches for obtaining online relaxations is to increase the set of benchmark solutions. The latter is a relaxation in the sense of optimization, as we show in the paper. Indeed, in this case, online relaxations and optimization relaxations are put on the same footing, and any distinction between the two should be clear from the context.

We use semidefinite relaxations in a somewhat unconventional way because the end goal is the problem of \emph{prediction}. The online relaxation requires us to compute the \emph{value} of the relaxed objective rather than the integer solution. Sidestepping the need to round the solution is a nice feature of ``improper'' prediction methods.  The integrality gap still comes into the picture, as it effectively quantifies the increase of Rademacher complexity for the larger set. Yet, the regret bound only requires \emph{existence} of a rounding procedure with a given guarantee and not its implementation. Crucially, the multiplicative increase due to the integrality gap is a constant that enters the regret bound only, leaving the constant in front of the comparator (OPT) to be one! The way in which the power of semidefinite relaxations fuses with the power of online relaxations is rather fortuitous.

The statements proved in this paper have an interesting ``modularity'' property. As soon as one finds a rounding procedure with a smaller integrality gap, this gap can be immediately inserted in the regret upper bound of our method. The prediction algorithm itself does not change, as it does not need to round the solution. Further, since Lasserre hierarchies we are employing are known to be tighter than LP-based and other hierarchies, the integrality gap can be proved for these weaker approximation methods.

We remark that it has been noted in the literature by various authors that the problem of prediction can be solved in situations when the offline solution is NP-hard (see e.g. \citep{HazKalSha12,christiano2014online,abernethy2010can}). Our work can be seen as formally extending this statement to approximation schemes, with an additional knob for the computation-prediction tradeoff. We also remark that ideas similar in spirit have been proposed in \citep{chandrasekaran2012convex,chandrasekaran2013computational}, among others, in the statistical (rather than online) setting. In particular, the recent paper of \cite{barak2015tensor} gives very strong guarantees for learning third-order tensors using the $6$th level of the sum-of-squares hierarchy. The authors compute a tight bound on the Rademacher complexity of the relaxed norm.

In summary, our contribution involves a framework for online prediction of labels for individuals that appear in a streaming fashion, with side information about individuals and constraints being also revealed in an online manner. The labels themselves can be adversarially chosen, while we assume that the stochastic model of the constraints and side information is known a priori. We propose a general method that is based on random playout, and further propose a semidefinite relaxation for the resulting CSP-like problem. We prove several regret bounds for the prediction method in terms of integrality gaps. The method allows for a trade-off between computation time and performance guarantee.

This paper is organized as follows. After describing the setting in the next section, we present in Section~\ref{sec:online_relax} the formalism of online relaxations and state a generic random-playout algorithm with a regret guarantee in terms of the expected relaxation. In Section~\ref{sec:rad_relax} we show that the relaxation based on classical Rademacher averages is ``admissible'', and we state the computationally-difficult problem. In Section~\ref{sec:sdp} we relax the problem in the SDP language of Lasserre hierarchy. Section~\ref{sec:existence_rounding} makes the connection between the integrality gap and the regret bound of the $r$-th level in the hierarchy. The main result here is Theorem~\ref{thm:mainround} which gives a regret bound in terms of the Rademacher complexity and the integrality gap. We turn to an alternative ``Lagrangian'' form of the optimization problem in Section~\ref{sec:regreg} and prove a regret bound for the $r$-th level of this form of relaxation (Theorem~\ref{thm:reg_version}). Several examples are discussed in Section~\ref{sec:examples}, and the paper is concluded with a lower bound in Section~\ref{sec:lower} which shows near-optimality of our methods in terms of prediction performance.

\paragraph{Notation} We use the following shorthand notation: let $[n]\deq \{1,\ldots,n\}$, $a_{1:t} \deq (a_1,\ldots,a_t)$, $(a,b)_{1:t} = (a_1,b_1,\ldots,a_t,b_t)$. We denote by $\Delta(A)$ the set of distributions on the set $A$.

\section{Setting}
\label{sec:setting}
On each round $t=1,\ldots,V$, the forecaster observes a new item along with side information $x_t \in \X_t \subseteq \X$ and a set $\constr_t$ of constraints. 
The forecaster then makes a prediction $\pred_t\in\{1,\ldots,\kappa\}\deq[\kappa]$ and observes the label $y_t\in[\kappa]$. The side information set $\X_t$ may be time-varying, but is known to the forecaster. Each constraint $c \in \constr_t$ is represented by a pair $(S_c, R_c)$ where $S_c  \subseteq \V$ and $R_c : [\kappa]^{S_c} \mapsto \reals_{\geq 0}$. For an assignment $g\in[\kappa]^V$, we write $c(g)$ or $R_c(g)$ for the value of $R_c$ on $g(S_c)$. To lighten the notation, let us introduce a shorthand $\cI_t = (\constr_t,x_t)$ for the associated constraints and the side information for the item.

\begin{example} 
	\label{ex:1}
	Let $\kappa=2$ and let $g\in\{1,2\}^V$ be an assignment of binary labels to vertices of an unweighted graph $G=(\V,E)$. Define a constraint $c$ for each edge $(u,v)\in E$ by taking $S_c=(u,v)$ and $R_c(g_u,g_v)=\ind{g_u \neq g_v}$. Any labeling $g$ defines a partition of $G$, and the size of the cut is precisely $\sum_{c} c(g)$.
\end{example}

Let $\F$ be a class of functions $\X\to [\kappa]$. Each $f\in\F$ gives rise to a vector $(f(x_1),\ldots,f(x_V))$ of labelings of the items. Given $x_1,\ldots,x_V$, each $f\in\F$ induces an assignment vector $[f(x_{j})]_{j=1}^V \in [\kappa]^V$, and now $c([f(x_{j})]_{j=1}^V)$ represents the value of the constraint $c$ on this assignment.

Let $\cup \constr_t=\cup_{t=1}^V \constr_t$ denote the union of all the constraint sets. Given this union, as well as $x_{1:V}$, we define the subset of those functions that do not violate more than $K$ constraints as
	\begin{align}
		\label{eq:def_class_unknown}
		\Fclass{K}{\cI_{1:V}} = \left\{f\in \F : \sum_{c\in \cup \constr_t} c\left([f(x_{1}) , \ldots, f(x_V)]\right) \leq K \right\}
	\end{align}
	for some given $K\geq 0$.
	
\begin{example}
	\label{ex:2}
	Continuing with Example~\ref{ex:1}, let $\F=\{f(x)=\ind{\inner{w,x}>\gamma}+1: w\in \reals^d\}$. The set in \eqref{eq:def_class_unknown} is then the set of homogenous hyperplanes that classify the vertices of the graph with a margin $\gamma$ in such a way that the cut is at most of size $K$.
\end{example}
	
Let $\loss(\pred_t,y_t) = \ind{\pred_t\neq y_t}$ be the indicator loss function. The goal of the forecaster is phrased as minimization of \emph{regret} 
\begin{align}
	\label{eq:regret_def}
	\Reg = \sum_{t=1}^V \loss(\pred_t,y_t) - \inf_{f \in \Fclass{K}{\cI_{1:V}}} \sum_{t=1}^V \loss(f(x_t),y_t)
\end{align}
with respect to the (data-dependent) subset of $\F$. This definition forces the forecaster to perform nearly as well as the benchmark that satisfies the constraints up to a certain threshold. 

We remark that the class $\F$ is ``pruned'' as more information about the constraints arrives over time. This pruning in effect captures the global information in the network, which requires adherence of labelings (given locally by values of $f$ on the side information) to the global structure of constraints. It is important to recognize that the forecaster faces a difficulty: the ``pruned'' set \eqref{eq:def_class_unknown} of comparators can only be calculated in hindsight. 

	We assume that the constraints and side information are drawn from a distribution known to the forecaster. That is, given $\cI_{1:t-1}$, we assume that the forecaster is able to draw samples from the conditional distributions 
	\begin{align}
		\label{eq:sampling_distributions}
		\bp(\constr_t, x_t| \cI_{1:t-1}).
	\end{align}

	
	\begin{example}[Preferential Attachment]
	In the preferential attachment model, the set $\constr_t$ of constraints corresponds to a set of new edges connected to previously revealed nodes. The edges are drawn according to the node degree given by the set of edges $\constr_{1:t-1}$. In this example, the distribution does not depend on side-information.
	\end{example}
	
	\begin{example}[Geometric Random Graphs]
	We may allow $x_t$'s to be drawn from some fixed distribution that does not depend on the constraints. In turn, the constraints can be formed according to the side information. One example is a geometric random graph, where pairwise constraints (graph edges) are formed according to distances from the new random point which may be given by the distance between the side information vectors. It is known that such graphs have better spectral properties \citep{barak2011subsampling}. The result in this paper indeed employ an average (rather than the worst-case) integrality gap and can take advantage of ``nice'' graphs. 
	\end{example}
	
	\begin{example}[Unlabeled Data]
		Rather than assuming the knowledge of the distribution of $x_t$'s, the random play-out algorithm introduced in the paper may tap into a pool of unlabeled data.
	\end{example}
		
	Other examples of distributions include a variant of the stochastic block model (SBM). This generative process provides the simplest model of group formation (though we remark that we are not aiming to recovery a hidden labeling, which is the focus of much research on SBM).

%

	The upper bounds on regret obtained in this paper will also hold for an intermediate time horizon $n\leq V$. This ``anytime'' property follows from the fact that constraints are only added, and not deleted. If one is only concerned with regret at time $V$, the deletion is easy to incorporate in the model.

Finally, let us mention that much of prior literature on online prediction on graphs requires the knowledge of the graph from the beginning. When the order in which nodes are presented is given to us in advance the problem is readily modeled by our setting via $\X_t=\{t\}$. We then write $f(x_t) = f(t)$, precisely the notation for a static expert \citep{PLG}. On the other hand, the case when nodes are presented to us in adversarial fashion is not directly modeled by the presented setting. However, the algorithms presented here can be easily extended to such a scenario. Indeed, at every round $t$, we simply pick some prefixed order for remaining unseen nodes and make predictions assuming this is the order in which nodes will be presented. On similar lines as the inductive proof in \cite{cesa2011efficient}, we can show that the algorithm enjoys the same regret against an adversarial ordering of nodes as the algorithm would for the case when the order is known in advance.

In summary, we presented a flexible problem definition that models the arrival of items and the evolution of constraints. The model encapsulates local information about the items. The goal of the forecaster is phrased as a \emph{global} measure of coherence given all the information at the end of the day. The rest of the paper is focused on exhibiting randomized methods that provably minimize regret in this general framework. We also focus on the computational issues associated with making predictions.


\section{Online Relaxations}
\label{sec:online_relax}

The idea of online relaxations was studied in \citep{rakhlin2012relax} as a generic recipe for deriving prediction algorithms. The basic technique for our context is as follows. Consider for a moment the problem that does not involve constraints, and suppose $x_1,\ldots,x_V$ are provided to the forecaster ahead of time. At time $t$, the forecaster predicts $\pred_t\in\Y$ and observes $y_t\in\Y$. Furthermore, suppose the comparator set $\G$ of functions $\X\to\Y$ in the regret definition is fixed. Given a loss function $\loss: \Y \times \Y \mapsto \reals$, an online relaxation $\mbf{Rel}$ is a sequence of functions that satisfies two conditions. First is the dominance condition: for any sequence of instances $x_{1:V}$ and $y_{1:V}$,
\begin{align}\label{eq:initial}
 \Relax{}{\G}{y_{1:V}} \ge - \inf_{f \in \G} \sum_{t=1}^V \loss(f(x_t),y_t).
\end{align}
Second is the recursive condition: for any $t\in[V]$,
\begin{align}\label{eq:admissibility}
\inf_{q_t \in \Delta(\Y)} \sup_{y_t \in \Y}\left\{\Es{\pred_t \sim q_t}{\loss(\pred_t,y_t)} + \Relax{}{\G}{y_{1:t}} \right\} \le \Relax{}{\G}{y_{1:t-1}}.
\end{align}
A relaxation that satisfies these conditions is termed \emph{admissible}. Given a relaxation $\mathbf{Rel}$ for a class $\G$, define an online learning algorithm which at time $t$, given instances $y_{1:t-1}$ and $x_{1:V}$, makes the random prediction $\pred_t$ by drawing from the distribution $q_t \in \Delta(\Y)$ either given by
$$
q_t = \argmin{q \in \Delta(\Y)} \sup_{y_t \in \Y}\left\{\Es{\pred_t \sim q}{ \loss(\pred_t,y_t)} + \Relax{}{\G}{y_{1:t}} \right\},
$$
or by any other choice that ensures admissibility of the relaxation. It can be easily shown that  regret of such a strategy is upper bounded (in expectation and with high probability) by $\Es{}{\Relax{}{\G}{\emptyset}}$.

We now turn to the case of side-information and constraints being revealed to the forecaster sequentially. We would like to ``lift'' the admissibility technique to this situation. To start, assume that we have a relaxation that is admissible for any class $\G=\Fclass{K}{\cI_{1:V}}$. We propose the following simple randomized strategy. 

\frameit{
	At time $t$, given $\cI_{1:t} = (\constr_s,x_s)_{s=1}^t$,  draw  $\cI_{t+1:V} = (\constr,x)_{t+1:V}$ from the known distribution $\bp$. Pick distribution $q_t$ over $\Y$ as follows
	\begin{align}
		\label{eq:relalgo_stoch}
		\widehat{q}_t(\cI_{t+1:V}) = \argmin{q \in \Delta(\Y)} \sup_{y_t \in \Y}\left\{ \Es{\pred_t \sim q}{\loss(\pred_t,y_t)} + \Relax{}{\Fclass{K}{\cI_{1:V}}}{y_{1:t}}\right\} 
	\end{align}
	and make a randomized prediction according to $\widehat{q}_t(\cI_{t+1:V})$.
}

As mentioned in the introduction, the above randomized method is of a ``random playout'' style. The forecaster simulates future draws to solve the (otherwise difficult) problem in expectation. The next lemma guarantees a bound on the expected regret in terms of expected Rademacher complexity of the data-dependent class. The upper bound behaves as if the forecaster were able to integrate over the complete distribution $\bp$ on each round, despite the fact that the method only draws one sample.

\begin{lemma}\label{lem:mainrel}
	Suppose $\mathbf{Rel}$ is an admissible relaxation for any $\Fclass{K}{\cI_{1:V}}$. Then the randomized algorithm given in \eqref{eq:relalgo_stoch} enjoys the performance guarantee
	\begin{align*}
	 \E{\Reg} \le \Es{(\constr,x)_{1:V}}{\Relax{}{\Fclass{K}{\cI_{1:V}}}{\emptyset}}
	\end{align*}
\end{lemma}
The proof of this lemma is postponed to the appendix. We refer to \citep{rakhlin2012relax} for more details of the technique.

Of course, the question remains: how do we come up with admissible relaxations required by Lemma~\ref{lem:mainrel}. This is the subject of the next section.


\section{Rademacher-Based Relaxations}
\label{sec:rad_relax}

The previous section presented a generic randomized prediction algorithm when the forecaster can sample from the distribution $\bp$ that generates the constraint sets and the side information. In this section, we provide a specific form of the relaxation we can use, along with the corresponding regret bound. The forecaster will be required to solve $\kappa$ optimization problems per round to obtain the randomized prediction for that round.

Let $\mc{M}$ be a set of $V \times \kappa$ matrices such that for any $M \in \mathcal{M}$, every $t \in [V]$ and $k \in [\kappa]$, $M_{t,k} \in [0,1]$ and $\sum_{k =1}^\kappa M_{t,k} \le 1$. Given any class $\G$ of functions $\X\to [\kappa]$ and side information $x_{1:V}$, we define a set of matrices $\mc{M}_\G$ as
$$
\mc{M}_\G = \{M_f : f \in \G, M_{t,k} = \ind{f(x_t) = k}\}.
$$
If $\kappa=2$, each $M_f$ can be simply represented by a vector of binary labels that $f$ assigns to $x_1,\ldots,x_V$.

\begin{lemma}\label{lem:mainoff}
For any class $\G$ of predictors, if $\mc{M}_\G \subseteq \mc{M}$, then the following relaxation is  admissible for prediction with respect to class $\G$:
 \begin{align*}
 \Relax{}{\mc{\G}}{y_{1:t}} =\Es{\bepsilon_{t+1:V}}{ \sup_{M \in \mc{M}} \left\{  2 \sum_{j=t+1}^V \sum_{k=1}^\kappa \bepsilon_{j,k} M_{j,k} + \sum_{i=1}^{t} M_{i,y_i} \right\}} - t \ .
\end{align*}
Here, each $\bepsilon_{j}$ is a vector of independent Rademacher random variables and $\bepsilon_{j,k}$ stands for the $k^{th}$ coordinate of this vector. Further, the randomized strategy corresponding to the above relaxation is given by first drawing $\bepsilon_{t+1:V}$ Rademacher vectors and then predicting $\pred_t$ according to
 $$
\widehat{q}_t(\bepsilon_{t+1:V}) = \argmin{q \in \Delta([\kappa])} \sup_{y_t \in [\kappa]}\left\{ 1 - q[y_t] + \sup_{M \in \mc{M}}\left\{  2 \sum_{j=t+1}^V \sum_{k =1}^\kappa \bepsilon_{j,k} M_{j,k} + \sum_{s=1}^{t} M_{s,y_s} \right\} - t\right\} . 
$$
\end{lemma}

Recall that Lemma \ref{lem:mainrel} provides a generic randomized strategy that, at round $t$, generates the future instances $\cI_{t+1:V}$ and then uses as a black box an  admissible relaxation for function classes  $\Fclass{K}{\cI_{1:V}}$. By combining Lemma \ref{lem:mainoff} and Lemma \ref{lem:mainrel}, we get the following randomized prediction strategy:  \\

\noindent At time $t$, given side information $x_{1:t}$, constraint sets $\constr_{1:t}$ and past labels $y_{1:t-1}$, draw $\cI_{t+1:V}$ from $\bp$. Next, draw Rademacher vectors $\bepsilon_{t+1:V}$ and compute, for each $o \in [\kappa]$, the value
\begin{align}\label{eq:RI}
R_t(o) =  \sup_{M \in \mc{M}(\cI_{1:V})}\left\{ 2 \sum_{j=t+1}^V \sum_{k =1}^\kappa \bepsilon_{j,k} M_{j,k} + M_{t,o} + \sum_{s=1}^{t-1} M_{s,y_s}  \right\}
\end{align}
where $\mc{M}(\cI_{1:V})$ is some set of matrices such that $\mc{\M}_{\Fclass{K}{\cI_{1:V}}} \subseteq \mc{M}(\cI_{1:V})$. Finally, we solve for the randomized strategy  $\widehat{q}_t(\bepsilon_{t+1:V})$ given by 
\begin{align}\label{eq:randpred}
 \widehat{q}_t(\bepsilon_{t+1:V}) = \argmin{q \in \Delta_\kappa}\max_{o \in [\kappa]}\left\{1 - q[o] +  R(o)\right\}.
\end{align}
Finally, predict $\pred_t$ by simply drawing it from $\widehat{q}_t(\bepsilon_{t+1:V})$.\\

Note that the step of solving for $\widehat{q}_t(\bepsilon_{t+1:V})$ can be done efficiently by first sorting $R_t(1),\ldots,R_t(\kappa)$'s in descending order and then using a simple water filling argument to find $\widehat{q}_t(\bepsilon_{t+1:V})$.

For the algorithm outlined above, in view of Lemma \ref{lem:mainrel}, the expected regret is upper-bounded as:
\begin{align}\label{eq:Radbound}
 \E{\Reg} \le 2\ \En_{(\constr,x)_{1:V}}\Es{\bepsilon_{1:V}}{\sup_{M \in \mc{M}(\cI_{1:V})}  \sum_{j=t}^V \sum_{k =1}^\kappa \bepsilon_{j,k} M_{j,k} }. 
\end{align}

Of course one could use $\mc{M}(\cI_{1:V}) = \mc{M}_{\Fclass{K}{\cI_{1:V}}}$. However in many prediction problems of interest, solving the optimization problem  (i.e., computing $R_t(o)$) for this class might be computationally hard. Hence, for computational efficiency we shall use a superset of $\mc{M}_{\Fclass{K}{\cI_{1:V}}}$. We pay for computational efficiency by having a worse regret bound given by the Rademacher complexity over the larger set  $\mc{M}(\cI_{1:V})$, rather than $\mc{M}_{\Fclass{K}{\cI_{1:V}}}$. We investigate this topic in the next two sections.


\section{Prediction Based on Lasserre SDP Hierarchy}
\label{sec:sdp}

In the previous section we provided a randomized prediction strategy based on any class of matrices $\mc{M}(\cI_{1:V})$ that is a superset of $\mc{M}_{\Fclass{K}{\cI_{1:V}}}$. 
In this section we will employ Semidefinite Programming and Lasserre hierarchies to solve for the values $R(o)$, defined in \eqref{eq:RI}. 

Let us begin with $\mc{M}_{\Fclass{K}{\cI_{1:V}}}$ and relax the problem. By the definition of $\mc{M}_{\Fclass{K}{\cI_{1:V}}}$, we can write down the optimization problem for each $o \in [\kappa]$ as
{\small \begin{align*}
&\max_{M \in \mc{M}_{\Fclass{K}{\cI_{1:V}}}}  \left\{ 2 \sum_{j=t+1}^V \sum_{k =1}^\kappa \bepsilon_{j,k} M_{j,k} +  M_{t,o} + \sum_{s=1}^{t-1} M_{s,y_s}   \right\} \\
&~~~~  = \max\left\{2 \sum_{j=t+1}^V \sum_{k =1}^\kappa \bepsilon_{j,k} M_{j,k}  + M_{t,o} + \sum_{s=1}^{t-1} M_{s,y_s} \right\} ~~~~~ \textrm{s.t. } \sum_{c\in \cup \constr_t} c\left(M\right) \leq K~,~~~~ M \in \F_{x_{1:V}} 
\end{align*}}
where, 
$$\F_{x_{1:V}} = \{ M \in \{0,1\}^{V \times \kappa}  : M_{t,i} = \ind{f(x_t) = i},~ f \in \F, t \in [V], i \in [\kappa]\}$$
We shall assume throughout this section that for any $x_{1:V}$, the set $\F_{x_{1:V}}$ can be represented as $\{0,1\}^{V \times \kappa} \cap \mc{P}^{x_{1:V}}$ where $\mc{P}^{x_{1:V}} \subset \reals^{V \times \kappa}$ can be represented by linear constraints efficiently. The superscript with side information is to remind us that the constraints can depend on the side information presented. To best match semidefinite formulations found in the literature, we assume,
$$
\mc{P}^{x_{1:V}} = \{ M \in \reals^{V \times \kappa} : \forall j \in [d],~~~ M^\top B^j  \le c_j\},
$$
an intersection of $d$ linear constraints. (Henceforth, whenever we refer to a matrix $M$ as a vector, we mean the vectorized form.) The reason for the assumption is that we would like to apply Lasserre Hierarchy to represent $\{0,1\}^{V \times \kappa} \cap \mc{P}$.
As an example, for the case of all possible static experts, we are interested in predicting as well as any labeling that violates at most $K$ constraints and, hence, $\mc{P}^{x_{1:V}}$ is simply $[0,1]^{V \times \kappa}$.

Given $y_{1:t-1}$ , $o \in [\kappa]$, and a draw of $\bepsilon_{t+1:V}$, we define the $V \times \kappa$ dimensional vector $Y^t(o)$ as 
{\small$$\textstyle Y^t_{s,j}(o) =
\begin{cases}
	\ind{j = y_s} &  s <t, ~ j \in[\kappa] \\
	2\bepsilon_{s,j} &  s > t,~ j \in [\kappa]\\
	\ind{j = o} & s=t,~  j \in [\kappa]
\end{cases}$$}
With this notation, we write the linear objective as
$$
2 \sum_{s=t+1}^V \sum_{k =1}^\kappa \bepsilon_{s,k} M_{s,k} + M_{t,o} + \sum_{s=1}^{t-1} M_{s,y_s}   = M^\top Y^t(o).
$$
We are now ready to write down the SDP relaxation that we shall solve for every round $t$ and every $o \in [\kappa]$ (these are the $R_t(o)$'s from \eqref{eq:RI}). The optimization problem is based on the $r^{th}$ level of Lasserre SDP relaxation, written is the vector form as follows. First, we introduce a vector $\mbf{U}_{S,\alpha}$ for every $S \subset [V]$ with $|S| \le r$ and every $\alpha \in [\kappa]^S$. The optimization problem is now written as
{\small\begin{align}\textstyle
\mrm{SDP}_r^{\mrm{1st}}& (Y, K)  =  \mathrm{max }~~~ A  \label{eq:SDP1} \\
 \textrm{s.t. } &  \sum_{c\in \cup \constr_t} \sum_{\alpha \in [q]^{S_c}} R_c(\alpha) \norm{\mbf{U}_{(S_c,\alpha)}}^2 \le K & \label{eq:SDP1mainconstraint}\\
&  \ip{\mbf{U}_{(S_1,\alpha_1)}}{\mbf{U}_{(S_2,\alpha_2)}} = 0 & \forall \alpha_1(S_1 \cap S_2) \ne \alpha_2(S_1 \cap S_2) \notag\\
& \ip{\mbf{U}_{(S_1,\alpha_1)}}{\mbf{U}_{(S_2,\alpha_2)}} = \ip{\mbf{U}_{(S_3,\alpha_3)}}{\mbf{U}_{(S_4,\alpha_4)}}  & \forall S_1 \cup S_2 = S_3 \cup S_4, \alpha_1 \circ \alpha_2 = \alpha_3 \circ \alpha_4\notag\\
& \sum_{k=1}^\kappa \norm{\mbf{U}_{(\{i\},k)}}^2 = 1 ,~ \norm{\mbf{U}_{\emptyset,\emptyset}}^2 = 1& \forall i \in [V]\notag\\
&\ip{\mbf{U}_{(S_1,\alpha_1)}}{\mbf{U}_{(S_2,\alpha_2)}} \ge 0 & \forall S_1,S_2, \alpha_1, \alpha_2\notag\\
& \sum_{\substack{v \in \V\\ \beta \in [\kappa]^v}} \norm{\mbf{U}_{(S \cup \{v\}, \alpha \circ \beta)}}^2 B^j_{(v,\beta)}  \le c_j \norm{\mbf{U}_{(S,\alpha)}}^2  & \forall S,  \alpha, j \in [d]\notag\\
& \sum_{\substack{v \in \V\\ \beta \in [\kappa]^v}} \norm{\mbf{U}_{(S \cup \{v\}, \alpha \circ \beta)}}^2 Y_{(v,\beta)}  \ge A \norm{\mbf{U}_{(S,\alpha)}}^2  & \forall S,  \alpha\label{eq:linobj}
 \end{align}}
where in the above $R_c \in [\kappa]^{S_c}$ is the constraint violation mapping corresponding to constraint $c$. The first constraint in the above program is the requirement that cumulative constraint violation does not exceed $K$. The rest of the constraints are standard (the notation $\alpha_1\circ \alpha_2$ denotes the concatenated assignment of labels whenever the assignments don't have a mismatch on the common entries). The above formulation is similar to the formulation for CSP's using Lasserre hierarchy, and we refer to \citep{Tulsiani09,raghavendra2009round,guruswami2013rounding,schoenebeck2008linear} for a more detailed treatment of the semidefinite relaxation technique.

In the above optimization problem, maximizing over $A$ can be performed efficiently as follows. First for a given $A$, we assume that we can solve the following optimization problem:
{\small\begin{align}
\mrm{SDP}_r^{\mrm{2nd}}& (Y, A)  =\ \mathrm{min }\sum_{c\in \cup \constr_t} \sum_{\alpha \in [q]^{S_c}} R_c(\alpha) \norm{\mbf{U}_{S_c,\alpha)}}^2 \label{eq:sdp2} \\
\textrm{under the constraints of} & ~~\mrm{SDP}_r^{\mrm{1st}} ~~\textrm{ excluding constraint }~~ \eqref{eq:SDP1mainconstraint}
\end{align}}
To find the solution to the maximization problem in \eqref{eq:SDP1} we simply perform a binary search over $A$ to find the largest $A$ for which the value of the solution of \eqref{eq:sdp2} is smaller than $K$.

On each round $t \in [V]$ and for each $o \in [\kappa]$, we find the value of  $\mrm{SDP}^{\mrm{1st}}(Y^t(o),K)$. This gives $R_t(o)$ in \eqref{eq:RI}, and, consequently, the randomized prediction obtained from \eqref{eq:randpred}. One can think of the solution in $\mc{M}(\cI_{1:V})$ as the projected solution from the $r^{th}$ level Lasserre hierarchy SDP. Specifically think of $\mc{M}(\cI_{1:V})$ as being described by set of vectors $\mbf{U}$ that satisfy the constraints of the SDP and $M_{j,k}$ as $\norm{U_{(\{j\},k)}}^2$.  It is important to note that for any constant level $r$, we obtain a poly-time algorithm. In the next session, we shall provide an analysis of the bound on the expected regret of this randomized strategy using the generic upper bound from \eqref{eq:Radbound}.


\section{Regret Bounds Based on \emph{Existence} of Rounding Strategies}
\label{sec:existence_rounding}

Let us define solutions to two other optimization problems in addition to the solution to  $\mrm{SDP}^{\mrm{1st}}(Y,K)$. These programs are defined for the purposes of analysis only, and will serve as a step to upper bounding Rademacher complexity of the relaxed set. To this end, define:

\begin{align}
\mrm{OPT}^{\mrm{2nd}}(Y,A) &=  \mathrm{min } \sum_{c\in \cup \constr_t} c(M) ~~~~~\textrm{subject to } & Y^\top M   \ge A, ~~~M \in \F_{x_{1:V}}  \label{eq:opt2}
 \end{align}
and
\begin{align}
\mrm{OPT}^{\mrm{1st}}(Y,K) & =  \mathrm{max }\ F^\top Y 
~~~~~\textrm{subject to } & \sum_{c\in \cup \constr_t} c(M) \leq K, ~~~ M \in \F_{x_{1:V}}\label{eq:opt1}
 \end{align}

\begin{definition}
Given $\cI_{1:V} = (\constr_t,x_t)_{1:V}$, we define the gap between the Lasserre SDP solution at level $r$  in \eqref{eq:sdp2} and the optimization problem in \eqref{eq:opt2} as
$$
\mathrm{gap}(r;\cI_{1:V}) := \sup_{\bepsilon \in \{-1,1\}^{V \times \kappa}, D \in [-V,V]}  \frac{\mrm{OPT}^{\mrm{2nd}}(\bepsilon, D)}{\mrm{SDP}_r^{\mrm{2nd}}(\bepsilon,D)} .
$$
Whenever the context of $\constr_{1:V},x_{1:V}$ is clear we will simply use $\mrm{gap}(r)$.
\end{definition}
The following theorem provides a bound on the expected regret of the proposed randomized strategy based on $\mathrm{gap}$. Observe that the regret bound only gains a multiplicative factor $\mathrm{gap}(r)$ in the constraint $K$, as compared to the original class. Below we prove our main theorem providing a bound on the expected regret of the proposed strategy in terms of the Rademacher complexity of the original class with its violation budget $K$ enlarged. For notational convenience given sequence $(\constr,x)_{1:V}$ and any $K>0$, let 
$$
\mathrm{Rad}_V(\Fclass{K}{\cI_{1:V}}) := \Es{\bepsilon_{1:V}}{\sup_{f \in\Fclass{K}{\cI_{1:V}}}  \sum_{j=1}^V \sum_{k =1}^\kappa \bepsilon_{j,k} \ind{f(x_j) = k} } 
$$
The following theorem is a performance guarantee for the proposed prediction strategy.

\begin{theorem}\label{thm:mainround}
If we use the $r^{th}$ level Lasserre hierarchy and use the randomized strategy obtained from the solutions via \eqref{eq:randpred}, the bound on the expected regret of the forecaster is given by 
\begin{align*}
 \E{\Reg} &\le 2\  \En_{\cI_{1:V}}\mathrm{Rad}_V(\Fclass{\mrm{gap}(r) \cdot K}{\cI_{1:V}})
 \end{align*}
\end{theorem}
\begin{proof}
From the bound in \eqref{eq:Radbound} we have that the expected regret of our algorithm is bounded as
{\small $$
 \E{\Reg} \le 2\ \En_{\cI_{1:V}}\Es{\bepsilon_{1:V}}{\sup_{M \in \mc{M}(\cI_{1:V})}  \sum_{j=t}^V \sum_{k =1}^\kappa \bepsilon_{j,k} M_{j,k} }. 
$$}
Let $M \in \mc{M}(\cI_{1:V})$ to be the projected solutions from the $r^{th}$ level Lasserre hierarchy SDP in the maximization problem in \eqref{eq:SDP1}.  Then for each draw of $\bepsilon_{1:V}$, the supremum in the Rademacher complexity term can be replaced by the value of the optimization problem in \eqref{eq:SDP1} given by $\mathrm{SDP}_r^{1st}(\bepsilon,K)$. This is because we can think of $M_{j,k}$ as corresponding to $\norm{\mbf{U}_{\{j\},k}}^{2}$ where vectors $\mbf{U}$'s satisfying constraints of the SDP.  On the other hand, for a given draw of $\bepsilon_{1:V}$, the solution to 
$$\sup_{f \in\Fclass{{\mathrm{gap}}(r) \cdot K}{\cI_{1:V}}}  \sum_{j=t}^V \sum_{k =1}^\kappa \bepsilon_{j,k} \ind{f(x_j) = k} $$ is exactly the value of $\mathrm{OPT}^{1st}(\bepsilon,{\mathrm{gap}}(r) \cdot K)$. Hence to prove our bound, it suffices to show that for any problem at hand, 
$$
\mathrm{SDP}_r^{1st}(\bepsilon,K) \le \mathrm{OPT}^{1st}(\bepsilon,{\mathrm{gap}}(r) \cdot K).
$$
To do so we go through the problems in Eqns. \eqref{eq:sdp2} and \eqref{eq:opt2} and arrive to $\mathrm{OPT}^{1st}(\bepsilon,{\mathrm{gap}}(r) \cdot K)$. Observe that the solution to the optimization problem in \eqref{eq:SDP1} is such that it has value $\mathrm{SDP}^{1st}(\bepsilon,K)$ and violates constraints by less than $K$. Using this feasible solution in \eqref{eq:sdp2} we conclude that,
$$
\mathrm{SDP}_r^{2nd}(\bepsilon,\mathrm{SDP}^{1st}(\bepsilon,K)) \le K
$$
However by definition of $\mathrm{gap}(r)$ we can conclude that 
$$ 
\mathrm{OPT}^{2nd}(\bepsilon,\mathrm{SDP}_r^{1st}(\bepsilon,K)) \le \mathrm{gap}(r) \cdot \mathrm{SDP}_r^{2nd}(\bepsilon,\mathrm{SDP}_r^{1st}(\bepsilon,K)) \le \mathrm{gap}(r) \cdot K
$$
By the definition of $\mathrm{OPT}^{2nd}$ this means that the solution $M \in \F_{x_{1:V}}$ to the optimization problem is such that $
\sum_{c\in \cup \constr_t} c(M)\le \mathrm{gap}(r) \cdot K$, and simultaneously, since we are considering $\mathrm{OPT}^{2nd}$ with second argument as $\mathrm{SDP}_r^{1st}(\bepsilon,K))$, $M^\top Y \ge \mathrm{SDP}_r^{1st}(\bepsilon,K)$. Thus by using this solution in the optimization problem in Eq. \eqref{eq:opt1} with second argument of $\mathrm{gap}(r) \cdot K$, we conclude:
$$
\mathrm{SDP}_r^{1st}(\bepsilon,K) \le \mathrm{OPT}^{1st}(\bepsilon,{\mathrm{gap}}(r) \cdot K)
$$
as required. Now since this is true for every $\bepsilon$,  we have that
\begin{align*}
  \E{\Reg} 
 & \le 2\  \En_{\cI_{1:V}}\Es{\bepsilon_{1:V}}{\sup_{f \in\Fclass{{\mathrm{gap}}(r) \cdot K}{\cI_{1:V}}}  \sum_{j=t}^V \sum_{k =1}^\kappa \bepsilon_{j,k} \ind{f(x_j) = k} }.
\end{align*}
\end{proof}

A few remarks are in order. First, since in the above $\mrm{gap}(r)$ really refers to $\mrm{gap}(r;\constr_{1:V},x_{1:V})$, for $\constr_{1:V},x_{1:V}$ drawn from the known generation process, bounds can often be improved: the behavior is given by the \emph{average} case gap rather than the worst case gap.

Second, we would like to stress that while the bounds in this section are provided in terms of integrality gaps, for the actual prediction algorithm we never require a rounding strategy. We only need existence of a rounding strategy with some integrality gap to provide bounds on the expected regret in terms of Rademacher complexity of the original class.

Third, as already mentioned in the introduction, the approximation factor multiplies the regret bound rather than the cumulative loss of the benchmark predictor. That is, regret is still with respect to $1 \times \mrm{OPT}$. As long as the integrality gap is not too large for $r = O(1)$ of the Lasserre hierarchy, we obtain polynomial-time algorithms even when the problem of finding the optimal benchmark predictor given all the instances and constraints might be computationally hard. This is due to the improper nature of the prediction algorithm.

The Lasserre hierarchy is known to be more powerful than the Sherali-Adams and Lovasz-Schrijver hierarchies. This means that if we use for our prediction strategy some $r \in \mathbb{N}$, then the $\mrm{gap}(r)$ we obtain is smaller than  approximation guarantees provided by algorithms using Sherali-Adams or Lov\'{a}sz-Schrijver hierarchies at around the same $r$. Also clearly $\mathrm{gap}(r) \le \mrm{gap}(r')$ for any $r' < r$. Thus we can use approximation guarantees proved for the same problems based on algorithms that use LP hierarchies at level $r$ or smaller. In summary, a result about an integrality gap for any weaker relaxation has immediate implication for the regret bound, without affecting the algorithm we use.

So far, we considered the problem where the benchmark was minimizing the number of violated constraints. Alternatively one could think of $\F$ being restricted across items by requiring that at least $K$ constraints need to be satisfied. Much of the machinery presented here including the application of rounding results to obtain bounds on the expected regret can easily be extended to such problems (which consist of typical CSP type problems) and in these cases the SDP optimization problems we solve on every step would be replaced by maximization versions of the SDP relaxations with the appropriate level of Lasserre hierarchy.


\section{Penalized Version of Relaxation}
\label{sec:regreg}

In this section we consider a penalized version of the relaxation, putting the ``$\leq K$'' constraint into the objective. We use the Lasserre hierarchy to solve the penalized version of the optimization problem. Let us write down the SDP corresponding to the $r^{th}$ level of Lasserre hierarchy. To this end, we introduce a vector $\mbf{U}_{S,\alpha}$ for every $S \subset [V]$ with $|S| \le r$ and every $\alpha \in [\kappa]^S$. The optimization problem is written as

{\small\begin{align}
\textstyle
\mrm{SDP}_r^{\lambda} (Y, \lambda)  =  &\min\left\{ \lambda \sum_{c\in \cup \constr_t} \sum_{\alpha \in [q]^{S_c}} R_c(\alpha) \norm{\mbf{U}_{S_c,\alpha)}}^2 - \sum_{\substack{v \in \V\\ \beta \in [\kappa]^v}}\norm{\mbf{U}_{(\{v\},  \beta)}}^2 Y_{(v,\beta)} \right\}  \label{eq:SDPreg} 
\end{align}
\begin{align}
 \textrm{s.t. } & \ip{\mbf{U}_{(S_1,\alpha_1)}}{\mbf{U}_{(S_2,\alpha_2)}} = 0  & \forall \alpha_1(S_1 \cap S_2) \ne \alpha_2(S_1 \cap S_2) \notag\\
& \ip{\mbf{U}_{(S_1,\alpha_1)}}{\mbf{U}_{(S_2,\alpha_2)}} = \ip{\mbf{U}_{(S_3,\alpha_3)}}{\mbf{U}_{(S_4,\alpha_4)}}  & \forall S_1 \cup S_2 = S_3 \cup S_4, \alpha_1 \circ \alpha_2 = \alpha_3 \circ \alpha_4\notag\\
& \sum_{k=1}^\kappa \norm{\mbf{U}_{(\{i\},k)}}^2 = 1 ,~ \norm{\mbf{U}_{\emptyset,\emptyset}}^2 = 1& \forall i \in [V]\notag\\
&\ip{\mbf{U}_{(S_1,\alpha_1)}}{\mbf{U}_{(S_2,\alpha_2)}} \ge 0 & \forall S_1,S_2, \alpha_1, \alpha_2\notag\\
& \sum_{\substack{v \in \V\\ \beta \in [\kappa]^v}} \norm{\mbf{U}_{(S \cup \{v\}, \alpha \circ \beta)}}^2 B^j_{(v,\beta)}  \le c_j \norm{\mbf{U}_{(S,\alpha)}}^2  & \forall S,  \alpha, j \in [d]\notag
 \end{align}}
 
This SDP should be compared to $\mrm{SDP}_r^{\mrm{1st}}$. Notice that the constraint \eqref{eq:SDP1mainconstraint} now appears in the objective.
We now prove a ``penalized version'' of Lemma \ref{lem:mainoff}. We will also provide an appropriate relaxation from which an efficient prediction strategy follows.

Let us define a slightly modified version of gap between the SDP solution and integral solution to the penalized optimization problem as follows. Define the optimization problem
\begin{align}
\mrm{OPT}^{\lambda}(Y,\lambda) =& \ \mathrm{min}\ \lambda \sum_{c\in \cup \constr_t} c(M) - Y^\top M \notag\\
& ~~\textrm{s.t.}~~~M \in \F_{x_{1:V}}  \label{eq:opt3}
 \end{align}

\begin{definition}
Given $(\constr_{1:V},x_{1:V})$, we define the gap between the Lasserre SDP solution at level $r$ in   \eqref{eq:SDPreg} and the optimization problem in \eqref{eq:opt3} as
$$
\widetilde\gap(r;\constr_{1:V},x_{1:V}) :=  \min\left\{a ~:~  \forall \bepsilon \in \{-1,1\}^{V \times \kappa},~~ \mrm{SDP}_r^{\lambda} (Y, \lambda) \ge \mrm{OPT}^{\lambda} (Y, \lambda/a)   \right\}
$$
Whenever the context of $\constr_{1:V},x_{1:V}$ is clear we will simply write $\widetilde\gap(r)$.
\end{definition}
That is the factor by which we only scale down the constraint costs but not the linear part.

\begin{lemma}\label{lem:mainoff_reg}
Given $(\constr,x)_{1:V}$, let $\G=\Fclass{K}{\cI_{1:V}}$, and fix any $\lambda>0$. Let $\Las(r,\F_{x_{1:V}})$ denote the set of vectors $\mbf{U}$'s corresponding to the $r$th level Lasserre hierarchy---that is, vectors satisfying the constraints of the SDP in Eq. \eqref{eq:SDPreg}. The following relaxation is admissible for prediction with respect to $\G$:
 {\small
 \begin{align*}
 \Relax{}{\mc{\G}}{y_{1:t}} = 
  \En_{\bepsilon_{t+1:V}}\sup_{\mbf{U} \in  \Las(r,\F_{x_{1:V}})}
 \Bigg\{  &2 \sum_{j=t+1}^V \sum_{k =1}^\kappa \bepsilon_{j,k} \norm{\mbf{U}_{(\{j\},k)}}^2 + \sum_{s=1}^{t} \norm{\mbf{U}_{(\{s\},y_s)}}^2  \\
 & ~~~- \lambda \sum_{c \in \constr_{1:V}} \sum_{\alpha \in [q]^{S_c}} R_c(\alpha) \norm{\mbf{U}_{(S_c,\alpha)}}^2 \Bigg\} - t  + \lambda K 
\end{align*}}
 Further, the randomized strategy corresponding to the above relaxation is given by first drawing $\bepsilon_{t+1:V}$ Rademacher vectors and then predicting $\pred_t$ according to
{\small 
\begin{align*}
\widehat{q}_t(\bepsilon_{t+1:V}) = \argmin{q \in \Delta([\kappa])} \sup_{y_t \in [\kappa]}\Bigg\{  \sup_{\mbf{U} \in  \Las(r,\F_{x_{1:V}})}
 \Bigg\{  &2 \sum_{j=t+1}^V \sum_{k =1}^\kappa \bepsilon_{j,k} \norm{\mbf{U}_{(\{j\},k)}}^2 + \sum_{s=1}^{t} \norm{\mbf{U}_{(\{s\},y_s)}}^2  \\
 &~~~- \lambda \sum_{c \in \constr_{1:V}} \sum_{\alpha \in [q]^{S_c}} R_c(\alpha) \norm{\mbf{U}_{(S_c,\alpha)}}^2 \Bigg\} - q[y_t] \Bigg\} . 
\end{align*}}
\end{lemma}
As before, each $\bepsilon_{j}$ is a vector of independent Rademacher random variables and $\bepsilon_{j,k}$ stands for the $k^{th}$ coordinate of this vector.

Let us bound the regret of the algorithm. To this end, assume we have a bound on the $\mrm{gap}$ for the penalized SDP.
\begin{theorem}
	\label{thm:reg_version}
	 Suppose that for any $c\geq 1$,
	 $$\mathrm{Rad}_V(\Fclass{c K }{\cI_{1:V}})\leq c^p \mathrm{Rad}_V(\Fclass{K }{\cI_{1:V}})$$
	 for some $p\leq 1$. With the notation of Lemma~\ref{lem:mainoff_reg}, if we choose
	{\small \begin{align}
		 \label{eq:opt_lambda}
		 \lambda^* = \sup\left\{ \lambda : \lambda K \le  \Es{\bepsilon_{1:V}}{ \sup_{\mbf{U} \in  \Las(r,\F_{x_{1:V}})}\left\{ 2 \sum_{t=1}^V \sum_{k=1}^\kappa \bepsilon_{j,k} \norm{\mbf{U}_{\{j\},k}}^2  - \lambda \sum_{c \in \cup_t \constr_t} \sum_{\alpha \in [q]^{S_c}} R_c(\alpha) \norm{\mbf{U}_{(S_c,\alpha)}}^2 \right\}}\right\},
	\end{align}}
	the final relaxation is upper bounded by
	\begin{align*}
	 \Relax{}{\G}{\emptyset} & \le 4\  \widetilde\gap(r) \ \mathrm{Rad}_V(\Fclass{K}{\cI_{1:V}}). 
	\end{align*}
	In view of Lemma~\ref{lem:mainrel}, the expected regret of the strategy described in \eqref{eq:relalgo_stoch} is upper bounded as
	\begin{align*}
	 \E{\Reg} & \le 4\  \widetilde\gap(r) \  \En_{\cI_{1:V}} \mathrm{Rad}_V(\Fclass{K}{\cI_{1:V}}).
	\end{align*}
\end{theorem}

To estimate $\lambda^*$ in \eqref{eq:opt_lambda}, we use the concentration property of Rademacher complexity. We sample Rademacher random variables, constraints, and side information. Next we optimize over the Lasserre SDP at level $r$ multiple times to find the maximal $\lambda$ that satisfies the inequality
$$
 \lambda K \le \sup_{\mbf{U} \in  \Las(r,\F_{x_{1:V}})}\left\{ 2 \sum_{t=1}^V \sum_{k =1}^\kappa \bepsilon_{j,k} \norm{\mbf{U}_{\{t\},k}}^2  - \lambda \sum_{c \in \cup_t \constr_t} \sum_{\alpha \in [q]^{S_c}} R_c(\alpha) \norm{\mbf{U}_{(S_c,\alpha)}}^2 \right\}.
$$


\section{Examples}
\label{sec:examples}

We illustrate the results of this paper on two examples. The first uses the SDP formulation in Section~\ref{sec:sdp}, while the second example uses the penalized version of Section~\ref{sec:regreg}.

\paragraph{Binary Classification of Nodes with Cut Constraints}
Let us consider a weighted version of the problem discussed in the introduction. Suppose we are given a weighted graph $G = (\V,E,W)$ where $W:E \mapsto [-1,1]$. Let us consider the case of $x_t = t$ and there is no other side information. The benchmark class of predictors is all binary labelings of the nodes from the set $\F_K=\{f\in [2]^V: f^\tr L f\leq K\}$, where $L$ is graph Laplacian. This problem can be formulated easily in the generic form specified in this paper by adding one constraint $c$ per edge $(u,v) \in E$ with $S_c = \{u,v\}$. The cost of the constraint violation is given by  $$R_c(\alpha) = 1 - W(e_{u,v}) (2\ \ind{\alpha(u) = \alpha(v)}-1).$$ 
These constraints can in fact be rewritten as quadratic constraints and Lasserre SDP at level $r$ for the $\mrm{SDP}_r^{\mrm{2nd}}$ problem in Eq. \eqref{eq:sdp2} is in fact the $r^{th}$ level SDP relaxation to the quadratic integer programming with a single linear constraint given by the labels (the one corresponding to \eqref{eq:linobj} in $\mrm{SDP}_r^{\mrm{1st}}$).

It is shown in \citep{guruswami2013rounding} that the value of a rounded solution with $O(r)$ levels of Lasserre hierarchy is no more than $2/\lambda_r(L)$ times OPT. Furthermore, the rounding is faithful, and hence concentration bounds hold for linear constraints \citep[Thm 6.1]{guruswami2013rounding}. Since the linear constraints are given by Rademacher random variables, standard concentration results tell us that de-randomization does not violate the constraints by more than $O(\sqrt{V})$. By tracing through the proof of Theorem~\ref{thm:mainround}, one can see that this extra $O(\sqrt{V})$ factor comes out additively in the final bound on Rademacher complexity. Since this factor is of smaller order than Rademacher complexity itself, the bound is not affected. We conclude that
$$
  \E{\Reg}  \le O\left(  \En_{(\constr,x)_{1:V}}\mathrm{Rad}_V\left(\Fclass{\frac{2}{\min\{1 , \lambda_{r}\}} \cdot K}{\cI_{1:V}}\right)\right)
$$
where $\lambda_{r}$ is the $r^{th}$ smallest eigenvalue of the normalized Laplacian of the graph, and the algorithm runs in time $n^{O(r)}$. If the graph generation process is well behaved in terms of spectral values of the Laplacian---like in a preferential attachment model for the graph---then the bound we obtain is near optimal. As a crude upper bound on $$\mathrm{Rad}_V\left(\Fclass{\frac{2}{\min\{1 , \lambda_{r}\}} \cdot K}{\cI_{1:V}}\right)$$ one can use $$\sqrt{ K V \max\{1 ,\lambda_{r}^{-1} \} \log V}.$$

Beyond the binary prediction considered above, one can also analyze the problem of predicting one of $[\kappa]$ labels for each node of a graph. As an interesting set of constraints, one can consider the Unique-Games-type constraints for labelings of edges in the graph. As a benchmark we compare our cumulative loss to the cumulative loss of the labelings that violate at most $K$ of the labeling constraints on edges. Similar to the previous example, this problem can also we written with quadratic form for constraints. The integrality gap from \citep{guruswami2013rounding} yields  a bound on regret in terms of Rademacher complexity of the original class where the constraint $K$ is enlarged by factor of order $\max\{1 ,\lambda_{r}^{-1} \}.$ Here, the de-randomization procedure incurs an additional $O(\sqrt{\kappa V})$ violation of the constraints, which, again, does not affect the final bound of Rademacher complexity.

\paragraph{Online Prediction with Metric Labeling Constraints}

In the metric labeling problem \citep{kleinberg2002approximation}, one aims to assign one of $\kappa$ labels to each of the $V$ items, minimizing a combinatorial objective function consisting of two parts: assignment costs per item and separation costs based on pairs of items. This model subsumes MAP estimation in a Markov random field model. 

More precisely, let $G=(\V, E, W)$ be a weighted graph with $W:E\to[0,1]$. The cost of an assignment $g\in [\kappa]^V$ is written as
\begin{align}
	\label{eq:assignment_cost}
	\sum_{v\in [V]} d_1(v, g_v) + \sum_{(u,v)\in E} W(u,v) d_2(g_u,g_v)
\end{align}
where $d_2:[\kappa]\times[\kappa]\to \reals_{\geq 0}$ is a metric on the space of labels and $d_1:[V]\times [\kappa]\to \reals_{\geq 0}$ is a cost of assigning a particular label to the node. The function $d_2$ is a metric on the space of labels, and this distance is multiplied by the edge weight, encouraging ``similar'' items (high edge weight) to pay more for disagreeing labels. 

To map this setting into our notation, we define two types of constraints. The first type of a constraint $c$ is associated to a singleton set $S_c=\{v\}$ and cost $R_c(g) = d_1(v, g_v)$, for $g\in [\kappa]^V$. The second type corresponds to separation costs, and we define it through  $S_c=\{u,v\}$ and $R_c(g) = W(u,v) d_2(g_u,g_v)$ if $(u,v)$ is an edge, and $0$ otherwise. 

To exhibit a polynomial-time method with a provable regret bound, we turn to the penalized version of SDP, developed in Section~\ref{sec:regreg}. We observe that both \citep{kleinberg2002approximation} and \citep{chekuri2004linear} study linear relaxations of the integer program and prove integrality gaps which are based on the separation costs. Specifically, \citep{chekuri2004linear}
use a simple LP relaxation for the problem, and since Lassere hierarchy at any level $r \ge 1$ is strictly stronger than this Linear program, we can directly use the integrality gap from \citep{chekuri2004linear} to obtain our regret bound. More precisely, \citep{chekuri2004linear} shows that the integrality gap for the separation costs is $O(\log \kappa)$, while the assignment costs are exact and have no integrality gap (gap of $1$). The overall integrality gap is then stated as $O(\log \kappa)$ by combining the two parts. However, for our purposes, it is important that the assignment costs are exact. To invoke the integrality gap result, we write the objective in \eqref{eq:SDPreg} as (negative of) the total cost \eqref{eq:assignment_cost} with the linear part involving $Y$ being incorporated into the assignment costs (per item). Since the values of $Y$ could be negative, we may only appeal to Theorem~\ref{thm:reg_version} if there is no gap for the assignment costs. This is the case for the proof in \citep{chekuri2004linear}, and we conclude that $$\widetilde\gap(r) = O(\log\kappa).$$
Theorem~\ref{thm:reg_version} then ensures a regret bound of Rademacher complexity of the class, increased multiplicatively by $O(\log\kappa)$.

The examples presented thus far extend to the case of having side information $x_t$, as long as the set $\F_{x_{1:V}}$ can be represented by polynomially-many constraints. One concrete example of when this can happen is if, for instance, we define
$$\F_{x_1,\ldots,x_V} = \left\{f\in[2]^V: \inf_{w\in B_\infty} \sum_{v\in[V]} |w^\tr x_v - (2f_v-3)| + \rho \sum_{(u,v) \in E} W_{u,v} d(f_u , f_v)  \leq  K \right\}.$$
The above class encodes a prior belief that the set of well-performing (in terms of prediction) labelings are close to those given by some linear function of side information.

Let us also mention an example where the constraints are defined in terms of the side information. Consider the above metric labeling problem, and imagine that the assignment cost $d_1(i_t, \cdot)$ is chosen according to $x_t$. We may use such a flexibility to provide a prior on the assignment of labels to individuals depending on the information about them.

We remark that the metric labeling objective subsumes Multiway Cut, among other problems. The objective also subsumes the energy function of the Ising model. Constraints based on Multiway Cut and Ising model appear to be well-suited for modeling global information dispersed throughout the graph. Furthermore, as soon as a better integrality gap is proved for a particular instance of a problem (such as, say, a known constant integrality gap for metric labeling on planar graphs), it can be immediately used in the regret bound without changing the algorithm.


\section{A Lower Bound}
\label{sec:lower}

In this short section we prove a lower bound, showing that the algorithms we developed are near-optimal in terms of regret guarantees. We first consider the case of binary classification with $\kappa=2$. We show a simple lower bound on the expected regret in terms of the Rademacher complexity of the constrained set of predictors. Next, we use the binary case lower bound to obtain a lower bound for the general case when $\kappa >2$. We show that the worst case regret of any prediction strategy is lower bounded by $1/\kappa$ times the Rademacher complexity. In summary, as long as the integrality gap is of constant order, and the Rademacher complexity of the class only depends polynomially on $K$, the upper bounds we obtained are optimal up to a constant factor indicated by the gap. 

\begin{proposition}\label{prop:lower}
For any $K$, any generating process that produces $(x_t,\constr_t)_{t=1}^V$ and any class of benchmark predictors $\F \subset [2]^\X$, there exists a strategy of labelings such that the following bound on the expected regret holds for any prediction algorithm:
$$
\E{\Reg} \ge \frac{1}{2} \En_{(\constr,x)_{1:V}}\mathrm{Rad}_V(\Fclass{ K}{\cI_{1:V}})
$$ 
\end{proposition}

\begin{corollary}\label{cor:lower}
 For any $K$, any generating process that produces $(x_t,\constr_t)_{t=1}^V$ and any class of benchmark predictors $\F \subset [\kappa]^\X$, there exists a strategy of labelings such that the following bound on the expected regret holds for any prediction algorithm:$$
\E{\Reg} \ge \frac{1}{\kappa} \En_{(\constr,x)_{1:V}}\mathrm{Rad}_V(\Fclass{ K}{\cI_{1:V}})
$$ 
\end{corollary}

\section*{Acknowledgements}{We thank David Steurer for many helpful discussions. We gratefully acknowledge the support of NSF under grants CAREER DMS-0954737 and CCF-1116928, ONR BRC Program on Decentralized, Online Optimization, as well as Dean's Research Fund.
}

\bibliography{paper}
\bibliographystyle{alpha}

\appendix
\section{Proofs}
\begin{proof}[\textbf{Proof of Lemma~\ref{lem:mainrel}}]
At time $t$, given $\{\constr_s,x_s\}_{s=1}^t$, let $q_t$ be a strategy defined by first drawing the random variables $\cI_{t+1:V} = (\constr,x)_{t+1:V}$  and then solving for the randomized strategy $\widehat{q}_t$ defined in \eqref{eq:relalgo_stoch}. We shall first prove the following inequality for any $t \in [V]$:
	\begin{align}\label{eq:induct}
 \underset{\constr_t,x_t}{\En}\left[\sup_{y_t}\left\{  \Es{\pred_t \sim q_t}{\loss(\pred_t,y_t)} + \Es{\cI_{t+1:V}}{\Relax{}{\Fclass{K}{\cI_{1:V}}}{y_{1:t}}} \right\}\right]  \le \Es{\cI_{t:V}}{\Relax{}{\Fclass{K}{\cI_{1:V}}}{y_{1:t-1}}}
	\end{align}
	 Here, the random variables $(\constr_s,x_s)$ follow the distribution given in \eqref{eq:sampling_distributions}.

	We will prove the above statement for any $t \in [V]$ by first starting from base case $t = V$ and then working backward inductively. To this end consider the very last step. Given $\constr_{1:V}, x_{1:V}, y_{1:V-1}$,
	\begin{align*}
	\sup_{y_V}&\left\{  \Es{\pred_V \sim q_V}{\loss(\pred_V,y_V)} + \Relax{}{\Fclass{K}{\cI_{1:V}}}{y_{1:V}} \right\} \\
	& = \inf_{q_V} \sup_{y_V}\left\{  \Es{\pred_V \sim q_V}{\loss(\pred_V,y_V)} + \Relax{}{\Fclass{K}{\cI_{1:V}}}{y_{1:V}} \right\}  \le \Relax{}{\Fclass{K}{\cI_{1:V}}}{y_{1:V-1}} 
	\end{align*}
	where the last inequality is by admissibility condition of the relaxation. Hence, we conclude that
	\begin{align*}
 \underset{\constr_V,x_V}{\En}\left[\sup_{y_V}\left\{  \Es{\pred_V \sim q_V}{\loss(\pred_V,y_V)} + \Relax{}{\Fclass{K}{\cI_{1:V}}}{y_{1:V}} \right\}\right] \le \underset{\constr_V,x_V}{\En}\left[\Relax{}{\Fclass{K}{\cI_{1:V}}}{y_{1:V-1}}\right]
	\end{align*}
	This proves the base case. Now assume the statement holds for any $\tau > t$ and let us conclude the  statement for $t$. For the $t^{th}$ round, given $\constr_{1:t}, x_{1:t}, y_{1:t-1}$, 
	\begin{align*}
	 \sup_{y_{t}}& \left\{  \Es{\pred_t \sim q_{t}}{\loss(\pred_{t},y_{t})}  + \Es{\cI_{t+1:V}}{\Relax{}{\Fclass{K}{\cI_{1:V}}}{y_{1:t}}} \right\}\\
	& = \sup_{y_{t}} \Bigg\{  \Es{\cI_{t+1:V}}{\Es{\pred_t \sim \widehat{q}_{t}(\cI_{t+1:V}) }{\loss(\pred_{t},y_{t})}} + \Es{\cI_{t+1:V}}{\Relax{}{\Fclass{K}{\cI_{1:V}}}{y_{1:t}}} \Bigg\}\\
	& \le  \Es{\cI_{t+1:V}}{ \sup_{y_{t}}\left\{  \Es{\pred_t \sim \widehat{q}_t(\cI_{t+1:V}) }{\loss(\pred_{t},y_{t})}  + \Relax{}{\Fclass{K}{\cI_{1:V}}}{y_{1:t}} \right\}}
\intertext{	By definition of $\widehat{q}_t$, the above expression is equal to}
&=	\Es{\cI_{t+1:V}}{ \inf_{q_t} \sup_{y_{t}}\left\{  \Es{\pred_t \sim q_t }{\loss(\pred_{t},y_{t})}  + \Relax{}{\Fclass{K}{\cI_{1:V}}}{y_{1:t}} \right\}}  \\
& \le \Es{\cI_{t+1:V}}{\Relax{}{\Fclass{K}{\cI_{1:V}}}{y_{1:t-1}} }
	\end{align*}
	Thus we can conclude that,
	\begin{align*}
\En_{\cI_t} \sup_{y_{t}}\left\{  \Es{\pred_t \sim \widehat{q}_{t}(\cI_{t+1:V}) }{\loss(\pred_{t},y_{t})}  +   \Es{\cI_{t+1:V}}{\Relax{}{\Fclass{K}{\cI_{1:V}}}{y_{1:t}}}  \right\} &\le \Es{\cI_{t:V}}{\Relax{}{\Fclass{K}{\cI_{1:V}}}{y_{1:t-1}} } 
	\end{align*}
	This proves \eqref{eq:induct} via the inductive argument. To conclude the proof of the lemma,  note that by the dominance condition,
	\begin{align*}
	 \sum_{t=1}^V \loss(\pred_t,y_t) - \inf_{f \in \Fclass{K}{\cI_{1:V}}} \sum_{t=1}^V \loss(f(x_t),y_t) \le \sum_{t=1}^V \loss(\pred_t,y_t) + \Relax{}{\Fclass{K}{\cI_{1:V}}}{y_{1:V}}
	\end{align*}
	Using the above inequality and Eq. \eqref{eq:induct} we conclude that,
	\begin{align*}
	&\Es{\cI_{1:V}}{\sum_{t=1}^V \Es{\pred_t \sim \widehat{q}_t}{\loss(\pred_t,y_t)} - \inf_{f \in \Fclass{K}{\cI_{1:V}}} \sum_{t=1}^V \loss(f(x_t),y_t)} \\
	& ~~~~~ \le \Es{\cI_{1:V}}{\sum_{t=1}^V \Es{\pred_t \sim \widehat{q}_t}{\loss(\pred_t,y_t)} + \Relax{}{\Fclass{K}{\cI_{1:V}}}{y_{1:V}}} \\
	&~~~~~  \le \Es{\cI_{1:V}}{\sum_{t=1}^{V-1} \Es{\pred_t \sim \widehat{q}_t}{\loss(\pred_t,y_t)} +  \Es{\constr_{V},x_V}{\Relax{}{\Fclass{K}{\cI_{1:V}}}{y_{1:V-1}} } }\\
	&~~~~~  \le \Es{\cI_{1:V}}{\Relax{}{\F_K\left(\cI_{1:V}\right)}{\cdot} } 
	\end{align*}
	This concludes the proof of the lemma.
\end{proof}

\begin{proof}[\textbf{Proof of Lemma~\ref{lem:mainoff}}]
The initial dominance condition is satisfied, since
\begin{align*}
 \Relax{V}{\G}{y_{1:V}} &= \sup_{M \in \mc{M}} \sum_{s=1}^{V} M_{s,y_s} - V \ge \sup_{M \in \mc{M}_\G} \sum_{s=1}^{V} M_{s,y_s} - V \\
& = \sup_{M \in \mc{M}_\G} \sum_{s=1}^{V} (M_{s,y_s} -1)  = - \inf_{f \in \G}  \sum_{s=1}^{V} \ind{f(x_s) \ne y_s}. 
\end{align*}
Next we show the recursive admissibility condition for the randomized strategy provided in the lemma. To this end note that,
{\small\begin{align*}
\max_{y_t \in [\kappa]} & \left\{ \Es{\pred_t \sim \hat{q}_t}{\loss(\pred_t,y_t)} + \Relax{}{\G}{y_{1:t}} \right\}\\
&= \max_{y_t \in [\kappa]} \left\{1 - \Es{\pred_t \sim \hat{q}_t}{\ind{y_t = \pred_t}} + \En_{\bepsilon_{t+1:V}}\sup_{M \in \mc{M}}\left\{ \sum_{s=1}^{t} M_{s,y_s} +  2 \sum_{j=t+1}^V \sum_{k =1}^\kappa \bepsilon_{j,k} M_{j,k} \right\} - t  \right\} \\
&= \max_{y_t \in [\kappa]} \left\{ - \En_{\bepsilon_{t+1:V}}\Es{\pred_t \sim \hat{q}_t(\bepsilon_{t+1:V})}{\ind{y_t = \pred_t}} + \En_{\bepsilon_{t+1:V}}\sup_{M \in \mc{M}}\left\{ \sum_{s=1}^{t} M_{s,y_s} +  2 \sum_{j=t+1}^V \sum_{k =1}^\kappa \bepsilon_{j,k} M_{j,k} \right\} - (t-1)  \right\} \\
&\le \Es{\bepsilon_{t+1:V}}{ \max_{y_t \in [\kappa]} \left\{ -\Es{\pred_t \sim \hat{q}_t(\bepsilon_{t+1:V})}{\ind{y_t = \pred_t}} + \sup_{M \in \mc{M}}\left\{ \sum_{s=1}^{t} M_{s,y_s} +  2 \sum_{j=t+1}^V \sum_{k =1}^\kappa \bepsilon_{j,k} M_{j,k} \right\} - (t-1) \right\} }
\end{align*}
By the definition of the randomized strategy, the last expression is equal to
\begin{align*}
&\Es{\bepsilon_{t+1:V}}{ \inf_{q_t \in \Delta([\kappa])}\max_{y_t \in [\kappa]} \left\{ -\Es{\pred_t \sim q_t}{\ind{y_t = \pred_t}} + \sup_{M \in \mc{M}}\left\{ \sum_{s=1}^{t} M_{s,y_s} +  2 \sum_{j=t+1}^V \sum_{k =1}^\kappa \bepsilon_{j,k} M_{j,k} \right\} - (t-1) \right\} }
\end{align*}
Using the minimax theorem, we can swap the infimum and supremum, and obtain equality to 
\begin{align*}
&\Es{\bepsilon_{t+1:V}}{ \sup_{p_t \in \Delta([\kappa])} \min_{\pred_t \in [\kappa]} \Es{y_t \sim p_t}{ - \ind{y_t = \pred_t} + \sup_{M \in \mc{M}}\left\{ \sum_{s=1}^{t} M_{s,y_s} +  2 \sum_{j=t+1}^V \sum_{k =1}^\kappa \bepsilon_{j,k} M_{j,k} \right\} } } - (t-1) \\
& =  \En_{\bepsilon_{t+1:V}}\sup_{p_t \in \Delta([\kappa])} \left\{ - \max_{\pred_t \in [\kappa]} \Es{y_t \sim p_t}{\ind{y_t = \pred_t}} + \Es{y_t \sim p_t}{\sup_{M \in \mc{M}}\left\{ \sum_{s=1}^{t} M_{s,y_s} +  2 \sum_{j=t+1}^V \sum_{k =1}^\kappa \bepsilon_{j,k} M_{j,k} \right\}  } \right\}  - (t-1) 
\end{align*}
Since
$$\underset{\pred_t \in [\kappa]}{\max} \underset{y_t \sim p_t}{\En}[\ind{y_t = \pred_t}] = \underset{i \in [\kappa]}{\max}\ p_t[i] \ge \underset{i \in [\kappa]}{\max}\ p_t[i]\left(\sum_{j} M_{t,j}\right) \ge \sum_{i} p_t[i] M_{t,i} = \underset{y'_t \sim p_t}{\En}[M_{t,y'_t}],$$
the previous expression can be upper bounded by
\begin{align*}
&\En_{\bepsilon_{t+1:V}}\sup_{p_t \in \Delta([\kappa])} \left\{ \Es{y_t \sim p_t}{\sup_{M \in \mc{M}}\left\{ \sum_{s=1}^{t-1} M_{s,y_s} + (M_{t,y_t} - \Es{y'_t \sim p_t}{M_{t,y'_t}}) +  2 \sum_{j=t+1}^V \sum_{k =1}^\kappa \bepsilon_{j,k} M_{j,k} \right\}  } \right\}  - (t-1) 
\end{align*}
which is upper bounded by Jensen's inequality by
\begin{align*}
&\En_{\bepsilon_{t+1:V}}\sup_{p_t \in \Delta([\kappa])} \left\{ \Es{y'_t, y_t \sim p_t}{\sup_{M \in \mc{M}}\left\{ \sum_{s=1}^{t-1} M_{s,y_s} + (M_{t,y_t} - M_{t,y'_t}) +  2 \sum_{j=t+1}^V \sum_{k =1}^\kappa \bepsilon_{j,k} M_{j,k} \right\}  } \right\}  - (t-1). 
\end{align*}
Since in above $y_t$ and $y'_t$ are identically distributed, we can introduce an independent Rademacher random variable $\delta_t$. The last expression is equal to
\begin{align*}
&\En_{\bepsilon_{t+1:V}}\sup_{p_t \in \Delta([\kappa])} \left\{ \Es{\delta_t, y'_t, y_t \sim p_t}{\sup_{M \in \mc{M}}\left\{ \sum_{s=1}^{t-1} M_{s,y_s} + \delta_t (M_{t,y_t} - M_{t,y'_t}) +  2 \sum_{j=t+1}^V \sum_{k =1}^\kappa \bepsilon_{j,k} M_{j,k} \right\}  } \right\}  - (t-1) \\
& \le  \En_{\bepsilon_{t+1:V}}\sup_{y_t, y'_t \in [\kappa]} \left\{ \Es{\delta_t}{\sup_{M \in \mc{M}}\left\{ \sum_{s=1}^{t-1} M_{s,y_s} + \delta_t (M_{t,y_t} - M_{t,y'_t}) +  2 \sum_{j=t+1}^V \sum_{k =1}^\kappa \bepsilon_{j,k} M_{j,k} \right\}  } \right\}  - (t-1) \\
& \le  \En_{\bepsilon_{t+1:V}}\sup_{y_t \in [\kappa]} \left\{ \Es{\delta_t}{\sup_{M \in \mc{M}}\left\{ \sum_{s=1}^{t-1} M_{s,y_s} + 2 \delta_t M_{t,y_t} +  2 \sum_{j=t+1}^V \sum_{k =1}^\kappa \bepsilon_{j,k} M_{j,k} \right\}  } \right\}  - (t-1) 
\end{align*}
Now let $\bepsilon^{y_t}_t\in\{\pm1\}^\kappa$ be defined as $1$ on coordinate $y_t$ and independent  Rademacher variables on the rest. For any $j \ne y_t$,
 $\E{\bepsilon^{y_t}_{t,j}} = 0$ and  $\bepsilon^{y_t}_{t,y_t} = 1$ and so the preceding expression is equal to
\begin{align*}
&\En_{\bepsilon_{t+1:V}}\sup_{y_t \in [\kappa]} \left\{ \Es{\delta_t}{\sup_{M \in \mc{M}}\left\{ \sum_{s=1}^{t-1} M_{s,y_s} + 2 \sum_{k=1}^\kappa \delta_t M_{t,k} \E{\bepsilon^{\y_t}_{t,k}} +  2 \sum_{j=t+1}^V \sum_{k =1}^\kappa \bepsilon_{j,k} M_{j,k} \right\}  } \right\}  - (t-1) \\
& \le \En_{\bepsilon_{t+1:V}}\sup_{y_t \in [\kappa]} \left\{ \Es{\delta_t, \bepsilon^{\y_t}_t}{\sup_{M \in \mc{M}}\left\{ \sum_{s=1}^{t-1} M_{s,y_s} + 2 \sum_{k=1}^\kappa \delta_t M_{t,k} \bepsilon^{\y_t}_{t,k} +  2 \sum_{j=t+1}^V \sum_{k =1}^\kappa \bepsilon_{j,k} M_{j,k} \right\}  } \right\}  - (t-1) \\
& = \En_{\bepsilon_{t+1:V}} \left\{ \Es{\bepsilon_t}{\sup_{M \in \mc{M}}\left\{ \sum_{s=1}^{t-1} M_{s,y_s} + 2 \sum_{k=1}^\kappa  M_{t,k} \bepsilon_{t,k} +  2 \sum_{j=t+1}^V \sum_{k =1}^\kappa \bepsilon_{j,k} M_{j,k} \right\}  } \right\}  - (t-1) \\
& = \Es{\bepsilon_{t:V}}{ \sup_{M \in \mc{M}}\left\{ \sum_{s=1}^{t-1} M_{s,y_s} +  2 \sum_{j=t}^V \sum_{k =1}^\kappa \bepsilon_{j,k} M_{j,k} \right\}  }   - (t-1) = \Relax{}{\G}{y_{1:t-1}}
\end{align*}}
Thus we have shown admissibility of the relaxation and demonstrated that the randomized strategy for the forecaster is given by the one in the lemma.
\end{proof}

\begin{proof}[\textbf{Proof of Proposition~\ref{prop:lower}}]
To prove the lower bound, we simply consider an adversary who picks nodes in the fixed sorted order and at each time step draw $\constr_t,x_t$ from the known generating process and finally draw $y_t \sim \mathrm{Unif}([\kappa])$. Now since $y_t$ is drawn independently and uniformly at random on every round, irrespective of how the forecaster picks $\hat{y}_t$, the expected loss of the forecaster is $\E{\ind{\hat{y}_t \ne y_t}} = 1/\kappa$. Thus we get the following lower bound on the expected regret.
\begin{align*}
\E{\Reg} & \ge \Es{(x_t,\constr_t)_{t=1}^V}{ \Es{y_{1:V} \sim \mathrm{Unif}([2])}{ V/2 - \inf_{f \in \Fclass{K}{\cI_{1:V}}} \sum_{t=1}^V \ind{f(x_t)\ne y_t} }}\\
& = \Es{\cI_{1:V}}{ \Es{y_{1:V} \sim \mathrm{Unif}([2])}{ \sup_{f \in \Fclass{K}{\cI_{1:V}}} \sum_{t=1}^V \left(\ind{f(x_t) = y_t}  - \frac{1}{2}\right) }}
\end{align*}
Now for the uniform distribution over $y_t$'s, since $\ind{f(x_t) = y_t}  - \frac{1}{2}$ and $\frac{1}{2} - \ind{f(x_t) = y_t} $ are identically distributed we see that,
\begin{align*}
\E{\Reg} & \ge \Es{(x_t,\constr_t)_{t=1}^V}{ \En_{y_{1:V} \sim \mathrm{Unif}([2])}\Es{\epsilon}{ \sup_{f \in \Fclass{K}{\cI_{1:V}}} \sum_{t=1}^V \epsilon_t\left(\ind{f(x_t) = y_t}  - \frac{1}{2}\right) }}\\
& = \Es{\cI_{1:V}}{ \En_{\boldsymbol{\epsilon}} \Es{y_{1:V} \sim \mathrm{Unif}([2])}{ \sup_{f \in \Fclass{K}{\cI_{1:V}}} \sum_{t=1}^V \boldsymbol{\epsilon}_{t,y_t}\ind{f(x_t) = y_t}  }}\\
& \ge \Es{\cI_{1:V}}{ \Es{\boldsymbol{\epsilon}}{ \sup_{f \in \Fclass{K}{\cI_{1:V}}} \sum_{t=1}^V \Es{y_{t} \sim \mathrm{Unif}([2])}{ \boldsymbol{\epsilon}_{t,y_t}\ind{f(x_t) = y_t}}  }}\\
& \ge \Es{\cI_{1:V}}{ \Es{\boldsymbol{\epsilon}}{ \sup_{f \in \Fclass{K}{\cI_{1:V}}} \sum_{t=1}^V \frac{1}{2} \sum_{k=1}^2 \left(\boldsymbol{\epsilon}_{t,k}\ind{f(x_t) = k}\right)}  }\\
& = \frac{1}{2} \Es{\cI_{1:V}, ~ \boldsymbol{\epsilon}}{ \sup_{f \in \Fclass{K}{\cI_{1:V}}} \sum_{t=1}^V \sum_{k=1}^2 \boldsymbol{\epsilon}_{t,k} \ind{f(x_t) = k}  }
\end{align*}
where the last line is because for any $f$ and any instance $x_t$ only one of $\ind{f(x_t)=1}$ or $\ind{f(x_t)=2}$ will be $1$ and the other is $0$. 
\end{proof}

\begin{proof}[\textbf{Proof of Corollary \ref{cor:lower}}]
This corollary follows by using a simple modification to Proposition ~\ref{prop:lower}. We shall assume here that $\kappa$ is even.  The simple modification is as follows: the adversary first picks uniformly at random a number $R$ from $[\kappa/2]$. Next the adversary uses exactly the lower bound construction as in Proposition ~\ref{prop:lower} except that instead of picking $y_t \sim \mrm{Unif}([2])$ the adversary picks $y_t \sim \mrm{Unif}(\{R,R+\kappa/2\})$. Now notice that given draw of $R$, this is exactly the binary case with labels $R$ and $R+\kappa/2$. Hence we can use the proposition to bound the expected regret as follows:
\begin{align*}
 \E{\Reg} &\ge \frac{1}{2} \underset{R \sim \mrm{Unif}([\kappa/2])}{\En}\Es{\cI_{1:V}, ~ \boldsymbol{\epsilon}}{ \sup_{f \in \Fclass{K}{\cI_{1:V}}} \sum_{t=1}^V \sum_{k \in \{R,R+\kappa/2\}} \boldsymbol{\epsilon}_{t,k} \ind{f(x_t) = k}  }\\
 & = \frac{1}{2} \underset{R \sim \mrm{Unif}([\kappa/2])}{\En}\Es{\cI_{1:V}, ~ \boldsymbol{\epsilon}}{ \sup_{f \in \Fclass{K}{\cI_{1:V}}} \sum_{t=1}^V \sum_{k =1}^\kappa \ind{k \in \{R,R+\kappa/2\}} \boldsymbol{\epsilon}_{t,k} \ind{f(x_t) = k}  }\\
 & \ge \frac{1}{2} \Es{\cI_{1:V}, ~ \boldsymbol{\epsilon}}{ \sup_{f \in \Fclass{K}{\cI_{1:V}}} \sum_{t=1}^V \sum_{k =1}^\kappa \underset{R \sim \mrm{Unif}([\kappa/2])}{\En}\left[\ind{k \in \{R,R+\kappa/2\}}\right] \boldsymbol{\epsilon}_{t,k} \ind{f(x_t) = k}  }\\
 & = \frac{1}{\kappa} \Es{\cI_{1:V}, ~ \boldsymbol{\epsilon}}{ \sup_{f \in \Fclass{K}{\cI_{1:V}}} \sum_{t=1}^V \sum_{k =1}^\kappa \boldsymbol{\epsilon}_{t,k} \ind{f(x_t) = k}  }
\end{align*}
\end{proof}


\begin{proof}[\textbf{Proof of Lemma~\ref{lem:mainoff_reg}}]
The proof closely follows the analogous proof of Lemma~\ref{lem:mainoff}. Note that we deal directly with the relaxed set of Lasserre's level $r$. To make the notation simpler, given a Lasserre vector set at level $r$, say $\mbf{U} \in \Las(r,\F_{x_{1:V}})$, let $M^{\mbf{U}}_{j,k} = \norm{\mbf{U}_{(\{j\},k)}}^2$ and also for each $t$ and each constraint $c \in \constr_t$ we use the notation 
$$
c(\mbf{U}) = \sum_{\alpha \in [q]^{S_c}} R_c(\alpha) \norm{\mbf{U}_{(S_c,\alpha)}}^2
$$
Now let us proceed to verify that the initial dominance condition is satisfied by the relaxation. Note that
\begin{align*}
- \inf_{f \in \G}  \sum_{s=1}^{V} \ind{f(x_s) \ne y_s} & \le - \inf_{f \in \G} \left\{ \sum_{s=1}^{V} \ind{f(x_s) \ne y_s} + \lambda \sum_{c \in \constr_{1:V}} c(f)\right\} + \lambda K\\
& \le - \inf_{\mbf{U} \in \Las(r,\F_{x_{1:V}})} \left\{ \sum_{s=1}^{V} M^{\mbf{U}}_{s,y_s} + \lambda \sum_{c \in \constr_{1:V}} c(\mbf{U})\right\} + \lambda K,
\end{align*}
where the first inequality holds because functions in $\G=\Fclass{K}{\cI_{1:V}}$ are required to keep the sum over unsatisfied constraints below $K$ by definition. The second inequality holds because the Lasserre solution is a relaxation of $\G$ and hence larger than the solution within $\G$. Let us check the recursive admissibility condition. To show that the proposed randomized strategy is admissible, we prove the recursive admissibility condition using this strategy directly:
{\small\begin{align*}
&\max_{y_t \in [\kappa]} \left\{ \Es{\pred_t \sim \hat{q}_t}{\loss(\pred_t,y_t)} + \Relax{}{\G}{y_{1:t}} \right\}\\
&= \max_{y_t \in [\kappa]} \left\{1 - \Es{\pred_t \sim \hat{q}_t}{\ind{y_t = \pred_t}} + \En_{\bepsilon_{t+1:V}}\sup_{\mbf{U} \in  \Las(r,\F_{x_{1:V}})}\left\{ \sum_{s=1}^{t} M^\mbf{U}_{s,y_s} +  2 \sum_{j=t+1}^V \sum_{k =1}^\kappa \bepsilon_{j,k} M^\mbf{U}_{j,k} - \lambda \sum_{c \in \constr_{1:V}} c(\mbf{U})\right\} \right\} \\
& ~~~~~~~~ - t  + \lambda K\\
&= \max_{y_t \in [\kappa]} \left\{ - \underset{\bepsilon_{t+1:V}}{\En}\underset{\pred_t \sim \hat{q}_t(\bepsilon_{t+1:V})}{\En}[\ind{y_t = \pred_t} + \underset{\bepsilon_{t+1:V}}{\En}\sup_{\mbf{U}  \in  \Las(r,\F_{x_{1:V}})}\left\{ \sum_{s=1}^{t} M^\mbf{U}_{s,y_s} +  2 \sum_{j=t+1}^V \sum_{k =1}^\kappa \bepsilon_{j,k} M^\mbf{U}_{j,k}  - \lambda \sum_{c \in \constr_{1:V}} c(\mbf{U}) \right\}  \right\} \\
& ~~~~~~~~ - (t-1) + \lambda K\\
&\le \underset{\bepsilon_{t+1:V}}{\En}\left[\max_{y_t \in [\kappa]} \left\{ -\underset{\pred_t \sim \hat{q}_t(\bepsilon_{t+1:V})}{\En}[\ind{y_t = \pred_t}] + \sup_{\mbf{U}  \in  \Las(r,\F_{x_{1:V}})}\left\{ \sum_{s=1}^{t} M^\mbf{U}_{s,y_s} +  2 \sum_{j=t+1}^V \sum_{k =1}^\kappa \bepsilon_{j,k} M^\mbf{U}_{j,k}   - \lambda \sum_{c \in \constr_{1:V}} c(\mbf{U})\right\}  \right\} \right]\\
& ~~~~~~~~  - (t-1) + \lambda K
\intertext{by the definition of the strategy, }
&= \underset{\bepsilon_{t+1:V}}{\En}\left[ \inf_{q_t \in \Delta([\kappa])}\max_{y_t \in [\kappa]} \left\{ -\Es{\pred_t \sim q_t}{\ind{y_t = \pred_t}} + \sup_{\mbf{U}  \in  \Las(r,\F_{x_{1:V}})}\left\{ \sum_{s=1}^{t} M^\mbf{U}_{s,y_s} +  2 \sum_{j=t+1}^V \sum_{k =1}^\kappa \bepsilon_{j,k} M^\mbf{U}_{j,k}  - \lambda \sum_{c \in \constr_{1:V}} c(\mbf{U})\right\}  \right\} \right]\\
& ~~~~~~~- (t-1) + \lambda K
\intertext{Using the minimax theorem, the above expression is equal to}
&= \Es{\bepsilon_{t+1:V}}{ \sup_{p_t \in \Delta([\kappa])} \min_{\pred_t \in [\kappa]} \Es{y_t \sim p_t}{ - \ind{y_t = \pred_t} + \sup_{\mbf{U} \in  \Las(r,\F_{x_{1:V}})}\left\{ \sum_{s=1}^{t} M^\mbf{U}_{s,y_s} +  2 \sum_{j=t+1}^V \sum_{k =1}^\kappa \bepsilon_{j,k} M^\mbf{U}_{j,k} - \lambda \sum_{c \in \constr_{1:V}} c(\mbf{U})\right\} } }  \\
& ~~~~~~~- (t-1) + \lambda K\\
& =  \underset{\bepsilon_{t+1:V}}{\En}\sup_{p_t \in \Delta([\kappa])} \left\{ - \max_{\pred_t \in [\kappa]} \underset{y_t \sim p_t}{\En}[\ind{y_t = \pred_t}] + \Es{y_t \sim p_t}{\sup_{\mbf{U} \in  \Las(r,\F_{x_{1:V}})} \left\{ \sum_{s=1}^{t} M^\mbf{U}_{s,y_s} +  2 \sum_{j=t+1}^V \sum_{k =1}^\kappa \bepsilon_{j,k} M^\mbf{U}_{j,k} - \lambda \sum_{c \in \constr_{1:V}} c(\mbf{U}) \right\}  } \right\} \\
& ~~~~~~~- (t-1) + \lambda K
\end{align*}
Once again, by the constraint in the SDP that for any $t$, $\sum_{k=1}^\kappa \norm{\mbf{U}_{(\{t\},k)}}^2 =\sum_{k=1}^\kappa M^{\mbf{U}}_{t,k} = 1$ we can conclude that
$$\underset{\pred_t \in [\kappa]}{\max} \underset{y_t \sim p_t}{\En}[\ind{y_t = \pred_t}] = \underset{i \in [\kappa]}{\max}\ p_t[i] \ge \underset{i \in [\kappa]}{\max}\ p_t[i]\left(\sum_{j} M^\mbf{U}_{t,j}\right) \ge \sum_{i} p_t[i] M^\mbf{U}_{t,i} = \underset{y'_t \sim p_t}{\En}[M^\mbf{U}_{t,y'_t}].$$
Hence, we conclude that,
\begin{align*}
& \max_{y_t \in [\kappa]} \left\{ \Es{\pred_t \sim \hat{q}_t}{\loss(\pred_t,y_t)} + \Relax{}{\G}{y_{1:t}} \right\} \\
& \le \underset{\bepsilon_{t+1:V}}{\En}\sup_{p_t \in \Delta([\kappa])} \left\{ \underset{y_t \sim p_t}{\En}\left[\sup_{\mbf{U} \in  \Las(r,\F_{x_{1:V}})}\left\{ \sum_{s=1}^{t-1} M^\mbf{U}_{s,y_s} + (M^\mbf{U}_{t,y_t} - \Es{y'_t \sim p_t}{M^\mbf{U}_{t,y'_t}}) +  2 \sum_{j=t+1}^V \sum_{k =1}^\kappa \bepsilon_{j,k} M^\mbf{U}_{j,k}   - \lambda \sum_{c \in \constr_{1:V}} c(\mbf{U}) \right\}  \right] \right\} \\
& ~~~~~~~- (t-1) + \lambda K\\
\intertext{using Jensen's inequality to pull out the expectation,}
& \le \En_{\bepsilon_{t+1:V}}\sup_{p_t \in \Delta([\kappa])} \left\{ \Es{y'_t, y_t \sim p_t}{\sup_{\mbf{U} \in  \Las(r,\F_{x_{1:V}})}\left\{ \sum_{s=1}^{t-1} M^\mbf{U}_{s,y_s} + (M^\mbf{U}_{t,y_t} - M^\mbf{U}_{t,y'_t}) +  2 \sum_{j=t+1}^V \sum_{k =1}^\kappa \bepsilon_{j,k} M^\mbf{U}_{j,k}  - \lambda \sum_{c \in \constr_{1:V}} c(\mbf{U})  \right\}  } \right\}  \\
& ~~~~~~~~ - (t-1) + \lambda K
\intertext{since $y_t$ and $y'_t$ are identically distributed, we can introduce Rademacher random variable $\delta_t$,}
&\En_{\bepsilon_{t+1:V}}\sup_{p_t \in \Delta([\kappa])} \left\{ \underset{\delta_t, y'_t, y_t \sim p_t}{\En}\left[\sup_{\mbf{U} \in \Las(r,\F_{x_{1:V}})}\left\{ \sum_{s=1}^{t-1} M^\mbf{U}_{s,y_s} + \delta_t (M^\mbf{U}_{t,y_t} - M^\mbf{U}_{t,y'_t}) +  2 \sum_{j=t+1}^V \sum_{k =1}^\kappa \bepsilon_{j,k} M^\mbf{U}_{j,k}  - \lambda \sum_{c \in \constr_{1:V}} c(\mbf{U}) \right\}  \right] \right\}  \\
& ~~~~~~~~~ - (t-1) +  \lambda K\\
& \le  \underset{\bepsilon_{t+1:V}}{\En}\sup_{y_t, y'_t \in [\kappa]} \left\{ \Es{\delta_t}{\sup_{\mbf{U} \in  \Las(r,\F_{x_{1:V}})}\left\{ \sum_{s=1}^{t-1} M^\mbf{U}_{s,y_s} + \delta_t (M^\mbf{U}_{t,y_t} - M^\mbf{U}_{t,y'_t}) +  2 \sum_{j=t+1}^V \sum_{k =1}^\kappa \bepsilon_{j,k} M^\mbf{U}_{j,k}   - \lambda \sum_{c \in \constr_{1:V}} c(\mbf{U})  \right\}  } \right\}  \\
& ~~~~~~~~~~-  (t-1) + \lambda K \\
& \le  \En_{\bepsilon_{t+1:V}}\sup_{y_t \in [\kappa]} \left\{ \Es{\delta_t}{\sup_{\mbf{U} \in  \Las(r,\F_{x_{1:V}})}\left\{ \sum_{s=1}^{t-1} M^\mbf{U}_{s,y_s} + 2 \delta_t M^\mbf{U}_{t,y_t} +  2 \sum_{j=t+1}^V \sum_{k =1}^\kappa \bepsilon_{j,k} M^\mbf{U}_{j,k}  - \lambda \sum_{c \in \constr_{1:V}} c(\mbf{U}) \right\}  } \right\}  \\
& ~~~~~~~~~~ - (t-1)  + \lambda K
\intertext{Let $\bepsilon^{y_t}_t\in\{\pm1\}^\kappa$ be defined as $1$ on coordinate $y_t$ and independent  Rademacher variables on the rest. For any $j \ne y_t$,  $\E{\bepsilon^{y_t}_{t,j}} = 0$ and  $\bepsilon^{y_t}_{t,y_t} = 1$ and so,}
& \le \En_{\bepsilon_{t+1:V}}\sup_{y_t \in [\kappa]} \left\{ \Es{\delta_t}{\sup_{\mbf{U} \in  \Las(r,\F_{x_{1:V}})}\left\{ \sum_{s=1}^{t-1} M^\mbf{U}_{s,y_s} + 2 \sum_{k=1}^\kappa \delta_t M^\mbf{U}_{t,k} \E{\bepsilon^{\y_t}_{t,k}} +  2 \sum_{j=t+1}^V \sum_{k =1}^\kappa \bepsilon_{j,k} M^\mbf{U}_{j,k}  - \lambda \sum_{c \in \constr_{1:V}} c(\mbf{U}) \right\}  } \right\}  \\
& ~~~~~~~~~~~- (t-1) + \lambda K\\
& \le \En_{\bepsilon_{t+1:V}}\sup_{y_t \in [\kappa]} \left\{ \Es{\delta_t, \bepsilon^{\y_t}_t}{\sup_{\mbf{U} \in  \Las(r,\F_{x_{1:V}})}\left\{ \sum_{s=1}^{t-1} M^\mbf{U}_{s,y_s} + 2 \sum_{k=1}^\kappa \delta_t M^\mbf{U}_{t,k} \bepsilon^{\y_t}_{t,k} +  2 \sum_{j=t+1}^V \sum_{k =1}^\kappa \bepsilon_{j,k} M^\mbf{U}_{j,k}  - \lambda \sum_{c \in \constr_{1:V}} c(\mbf{U}) \right\}  } \right\} \\
& ~~~~~~~~~~~  - (t-1) + \lambda K\\
& = \En_{\bepsilon_{t+1:V}} \left\{ \Es{\bepsilon_t}{\sup_{\mbf{U} \in  \Las(r,\F_{x_{1:V}})}\left\{ \sum_{s=1}^{t-1} M_{s,y_s} + 2 \sum_{k=1}^\kappa  M^\mbf{U}_{t,k} \bepsilon_{t,k} +  2 \sum_{j=t+1}^V \sum_{k =1}^\kappa \bepsilon_{j,k} M^\mbf{U}_{j,k}  - \lambda \sum_{c \in \constr_{1:V}} c(\mbf{U}) \right\}  } \right\}  \\
& ~~~~~~~~~~~~ - (t-1) + \lambda K\\
& = \Es{\bepsilon_{t:V}}{ \sup_{\mbf{U} \in  \Las(r,\F_{x_{1:V}})}\left\{ \sum_{s=1}^{t-1} M^\mbf{U}_{s,y_s} +  2 \sum_{j=t}^V \sum_{k =1}^\kappa \bepsilon_{j,k} M^\mbf{U}_{j,k}  - \lambda \sum_{c \in \constr_{1:V}} c(\mbf{U}) \right\}  }   - (t-1) + \lambda K \\
&= \Relax{}{\G}{y_{1:t-1}}
\end{align*}}
\end{proof}

\begin{proof}[\textbf{Proof of Theorem~\ref{thm:reg_version}}]
	We have
\begin{align*}
 \Relax{}{\G}{\emptyset} = & \lambda^* K  + \Es{\bepsilon_{1:V}}{ \sup_{\mbf{U} \in  \Las(r,\F_{x_{1:V}})}\left\{ 2 \sum_{t=1}^V \sum_{k =1}^\kappa \bepsilon_{t,k} \norm{\mbf{U}_{\{t\},k}}^2  - \lambda^* \sum_{c \in \cup_t \constr_t} \sum_{\alpha \in [q]^{S_c}} R_c(\alpha) \norm{\mbf{U}_{(S_c,\alpha)}}^2 \right\}  } \\
  & \le 2 \Es{\bepsilon_{1:V}}{ \sup_{\mbf{U} \in  \Las(r,\F_{x_{1:V}})}\left\{ 2 \sum_{t=1}^V \sum_{k =1}^\kappa \bepsilon_{t,k} \norm{\mbf{U}_{\{t\},k}}^2  - \lambda^* \sum_{c \in \cup_t \constr_t} \sum_{\alpha \in [q]^{S_c}} R_c(\alpha) \norm{\mbf{U}_{(S_c,\alpha)}}^2 \right\}  }\\ 
  &= 2 \lambda^* K .
\end{align*}
Now by definition of $\widetilde\gap(r)$ we conclude that
\begin{align*}
 \Relax{}{\G}{\emptyset} &\le 2 ~ \Es{\bepsilon_{1:V}}{ \sup_{M \in  \F_{x_{1:V}}}\left\{ 2 \sum_{t=1}^V \sum_{k =1}^\kappa \bepsilon_{t,k} M_{t,k}  - \frac{\lambda^*}{\widetilde\gap(r)} \sum_{c \in \cup_t \constr_t} c(M) \right\}  } 
\intertext{ Defining $K_i = 2^i$, we get an upper bound,}
 & \le 2 ~ \Es{\bepsilon_{1:V}}{ \max_{i \in \mathbb{Z}}\sup_{\substack{M \in \F_{x_{1:V}} \\ K_{i-1} \le \sum_{c \in \cup_t \constr_t} c(M) \le K_i}}\left\{ 2 \sum_{t=1}^V \sum_{k =1}^\kappa \bepsilon_{t,k} M_{t,k}  - \frac{\lambda^*}{\widetilde\gap(r)} \sum_{c \in \cup_t \constr_t} c(M) \right\}  }\\
	&\le 2 ~  \max_{i \in \mathbb{Z}}\left\{\Es{\bepsilon_{1:V}}{\sup_{\substack{M \in \F_{x_{1:V}} \\\sum_{c \in \cup_t \constr_t} c(M) \le K_i}}\left\{ 2 \sum_{t=1}^V \sum_{k =1}^\kappa \bepsilon_{j,k} M_{j,k} \right\}  } - \frac{\lambda^*}{\widetilde\gap(r)} K_{i-1} \right\}\\
& = 2 ~  \max_{i \in \mathbb{Z}}\left\{\ \mathrm{Rad}_V(\Fclass{K_{i}}{\cI_{1:V}})  - \frac{\lambda^*}{\widetilde\gap(r)} K_{i-1} \right\}\\
& =  2 ~  \max_{i \in \mathbb{Z}}\left\{\ \mathrm{Rad}_V(\Fclass{K_{i}}{\cI_{1:V}})  - \frac{\lambda^*}{2 \widetilde\gap(r)} K_{i}   \right\}\\
& \le  2 ~  \max_{i \in \mathbb{Z}}\left\{\ \mathrm{Rad}_V(\Fclass{\max\left\{1 , \frac{K_{i}}{K}\right\} K }{\cI_{1:V}})  - \frac{\lambda^*}{2~\widetilde\gap(r)} K_{i}   \right\}\\
\intertext{}
& \le  2 ~  \max_{i \in \mathbb{Z}}\left\{\ \max\left\{1 , \frac{K_{i}}{K}\right\}^{p}  \mathrm{Rad}_V(\Fclass{K}{\cI_{1:V}})  - \frac{\lambda^*}{2~\widetilde\gap(r)} K_{i}   \right\}.
\end{align*}
Now let us split the analysis into two cases. First, if $\lambda^* > \frac{2~\widetilde\gap(r)~ \mathrm{Rad}_V(\Fclass{K}{\cI_{1:V}})}{K}$, then
\begin{align*}
\Relax{}{\G}{\emptyset} & \le 2~  \max_{i \in \mbb{Z}}\left\{\max\left\{1 , \frac{K_{i}}{K}\right\}^{p}  \mathrm{Rad}_V(\Fclass{K}{\cI_{1:V}})  -  \frac{ K_{i}}{K} \mathrm{Rad}_V(\Fclass{K}{\cI_{1:V}})   \right\} \le 2~    \mathrm{Rad}_V(\Fclass{K}{\cI_{1:V}})  
\end{align*}
where the last line is because $p \le 1$. Next let us consider the case when $\lambda^* \le \frac{2 \mathrm{Rad}_V(\Fclass{K}{\cI_{1:V}})}{K~\widetilde\gap(r)}$. For this case however, note that we already showed that 
$\Relax{}{\G}{\emptyset} \le 2 \lambda^* K$ and so
\begin{align*}
\Relax{}{\G}{\emptyset} & \le 4\ \widetilde\gap(r)\ \mathrm{Rad}_V(\Fclass{K}{\cI_{1:V}}).
\end{align*}
The first statement follows. The second statement of the Theorem is an immediate consequence of Lemma~\ref{lem:mainrel}.
\end{proof}

\end{document}